\documentclass[9pt,twoside]{extarticle}

\usepackage[papersize={17cm,24cm},margin=22mm,top=17mm,headsep=5mm]{geometry} 
\usepackage[bottom]{footmisc}
\usepackage{caption}
\usepackage[utf8]{inputenc}
\usepackage{mathtools}
\usepackage{dsfont}
\usepackage{hyperref}
\usepackage{bm}
\usepackage{nicefrac}
\usepackage[dvipsnames]{xcolor}
\usepackage{wrapfig}
\usepackage{tikz}
\usepackage[colorinlistoftodos, textwidth=26mm, shadow,color=blue!30!white]{todonotes}
\usepackage{amsmath,amssymb,amsfonts}
\usepackage{algorithmic}
\usepackage[square, numbers]{natbib}
\usepackage[ruled,vlined]{algorithm2e}
\hypersetup{colorlinks,linkcolor = violet, urlcolor  = YellowOrange, citecolor = Aquamarine, anchorcolor = YellowOrange}
\makeatletter

\newcommand{\Rmnum}[1]{\expandafter\@slowromancap\romannumeral #1@}
\makeatother
\newcommand{\ph}{\ifmmode \text{\Rmnum{1}} \else \Rmnum{1} \fi}
\newcommand{\pr}{\ifmmode \text{\Rmnum{2}} \else \Rmnum{2} \fi}

\usepackage{fancyhdr}
\fancypagestyle{plain}{}
\fancypagestyle{plain}{
  \fancyhead[C,R]{}
  \fancyfoot{} 
  }

\newcommand{\alg}{Rob-ULA\xspace}
\newcommand{\N}{\mathbb N}

\newcounter{assumption}
\renewcommand{\theassumption}{\arabic{assumption}}

\newenvironment{assumption}[1][]{\begin{trivlist}\item[] \refstepcounter{assumption}%
 {\bf{Assumption\ \theassumption\ }}{(#1)}.  }{
 \ifvmode\smallskip\fi\end{trivlist}}

\newcommand{\acc}{\varepsilon_{\textsf{acc}}}

\newcommand{\rd}{\mathrm{d}}

\newcommand{\lmx}{\lambda_{\text{max}}}
\newcommand{\lmn}{\lambda_{\text{min}}}
\newcommand{\eup}{\bar{\epsilon}}
\newcommand{\elo}{\underline{\epsilon}}
\newcommand\ip[2]{\langle #1, #2 \rangle}
\newcommand{\wtS}{\widetilde{\Sigma}}
\newcommand{\ex}{\mathbb{E}}
\newcommand{\lip}{L}
\newcommand{\conv}{m}
\newcommand{\cond}{\kappa}
\newcommand{\covsc}{\beta}
\newcommand{\wass}{\text{W}_2}
\newcommand{\convav}{\bar{\conv}}
\newcommand{\lipav}{\bar{\lip}}
\newcommand{\pari}{\tilde{\param}}

\newcommand*\R[0]{\mathbb{R}}

\newcommand*\ddt[0]{\frac{d}{d t}}

\newcommand*\tr[0]{\text{tr}}

\newcommand*\lin[1]{\left\langle #1\right\rangle}
\newcommand*\E[1]{\mathbb{E}\left[#1\right]}
\newcommand*\Ep[2]{\mathbb{E}_{#1}\left[#2\right]}

\newcommand*\D[0]{\mathcal{D}}

\newcommand*\lrp[1]{\left(#1\right)}
\newcommand*\lrn[1]{\left\|#1\right\|}

\newcommand*{\p}{p}

\def\mI{\mathrm{I}}

\def\vtheta{{{ \theta}}}
\def\dtheta{\Delta_\vtheta}
\def\vTheta{{{\Theta}}}
\def\vzeta{{{\zeta}}}
\def\vx{{\bf x}}
\def\vX{\Theta}

\def\rd{\mathrm{d}}

\def\mI{\mathrm{I}}

\newcommand{\1}{\ensuremath{{\sf (i)}}}
\newcommand{\2}{\ensuremath{{\sf (ii)}}}

\usepackage{amssymb}
\newcommand*{\qed}{\hfill\ensuremath{\blacksquare}}%

\newcommand{\density}{p(\param)}

\newcommand{\target}{\ensuremath{p^*(\param)}}

\newcommand{\step}{h}

\newcommand{\real}{\ensuremath{\mathbb{R}}}

\DeclarePairedDelimiterX{\infdivx}[2]{(}{)}{%
  #1\;\delimsize\|\;#2%
}
\newcommand{\kldiv}{\text{KL}\infdivx}

\newcommand{\NORMAL}{\ensuremath{\mathcal{N}}}

\newcommand{\Ot}{\tilde{\mathcal{O}}}
\newcommand{\Pc}{P}
\newcommand{\Pa}{Q}

\newcommand{\parenth}[1]{\left( #1 \right)}

\newcommand{\param}{\theta}
\newcommand{\Param}{\Theta}
\newcommand{\paramb}{\nu}
\newcommand{\data}{\mathcal{D}}
\newcommand{\z}{z}
\newcommand{\x}{x}
\newcommand{\prior}{p_\param}
\newcommand{\hyper}{\alpha}
\newcommand{\tfunc}{h}
\newcommand{\pmean}{q}
\newcommand{\f}{f}
\newcommand{\stp}{\eta}
\newcommand{\grad}{\nabla}
\newcommand{\gnoise}{\xi}
\newcommand{\lang}{\theta}
\newcommand{\g}{g}
\newcommand{\Cr}{C_{R}}
\newcommand{\Csa}{C_{\small{\Sigma, 1}}}
\newcommand{\Csb}{C_{\small{\Sigma, 2}}}

\newcommand{\BlackBox}{\rule{1.5ex}{1.5ex}}  
\newenvironment{proof}{\par\noindent{\bf Proof\ }}{\hfill\BlackBox\\[2mm]}

\newtheorem{theorem}{Theorem}
\newtheorem{lemma}[theorem]{Lemma}

\newtheorem{corollary}[theorem]{Corollary}

\begin{document}

\title{\Huge{{Bayesian Robustness: A Nonasymptotic Viewpoint}}}

\author{
Kush Bhatia$^{\dagger}$\\
\texttt{kushbhatia@berkeley.edu}\\
\and
Yi-An Ma$^{\dagger}$\\
\texttt{yianma@berkeley.edu}\\
\and
Anca D. Dragan$^{\dagger}$\\
\texttt{anca@berkeley.edu }\\
\and
Peter L. Bartlett$^{\dagger, \ddagger}$\\
\texttt{peter@berkeley.edu}\\
\and
Michael I. Jordan$^{\dagger, \ddagger}$\\ \texttt{jordan@cs.berkeley.edu}\\
}

\date{%
    $\dagger$Department of Electrical Engineering and Computer Sciences\\%
    $\ddagger$Department of Statistics, University of California, Berkeley\\[2ex]%
    \today
}

\maketitle

\begin{abstract}
We study the problem of robustly estimating the posterior distribution for the setting where observed data can be contaminated with potentially adversarial outliers. We propose \alg, a robust variant of the Unadjusted Langevin Algorithm (ULA), and provide a finite-sample analysis of its sampling distribution. In particular, we show that after $T= \Ot(d/\acc)$ iterations, we can sample from $p_T$ such that $\text{dist}(p_T, p^*) \leq \acc + \Ot(\epsilon)$, where $\epsilon$ is the fraction of corruptions. We corroborate our theoretical analysis with experiments on both synthetic and real-world data sets for mean estimation, regression and binary classification.
\end{abstract}
\section{Introduction}\label{sec:intro}

Robustness has been of ongoing interest in both the Bayesian~\citep{de1961, berger1994} and frequentist setting~\citep{tukey1960, huber1964, huber1973} since being introduced by George Box in 1953 \citep{box1953}.
The goal is to capture the sensitivity of inferential procedures to the presence of outliers in the data and misspecifications in the modelling assumptions, and to mitigate overly large sensitivity. The Bayesian approach has been focused on capturing possible anomalies in the observed data via the model and in choosing priors that have minimal effect on inferences.
The frequentist approach, on the other hand, has focused on the development of estimators that identify and guard against outliers in the data.
We refer the reader to \cite[Chap 15]{huber2011} for a comprehensive discussion.



The focus on \emph{model robustness} in Bayesian statistics is implemented via sensitivity studies to understand effects of misspecification of the prior distribution~\citep{berger1994, minsker2017} and its propagation towards the posterior~\citep{huber1973b}. There is, however, little in the way of a comprehensive formal finite-sample framework for Bayesian robustness.  Huber asked ``Why there is no finite sample Bayesian robustness theory?" and Kong suggested that such a theory is infeasible in full generality, arguing that it is computationally infeasible to carry out the necessary calculations even in finite spaces.


We address this issue by providing a formal framework for studying Bayesian robustness and by proposing a robust inferential procedure with finite-sample guarantees. We address issues of computational infeasibility by refraining from modelling outlier data explicitly. Instead, we posit that the collected data contains a small fraction of observations which are not explained by the modelling assumptions. This corruption model, termed an $\epsilon$-contamination model, was first studied by Huber in 1964~\citep{huber1964} and has been the subject of recent computationally-focused work in the frequentist setting~\citep{diakonikolas2016, lai2016, prasad2018}.

Given data corrupted according to an $\epsilon$-contamination model, our goal is to sample from the \emph{clean} posterior distribution $p^*$: the posterior distribution conditioning only on the uncorrupted (``clean'') data. Our key idea is to leverage a robust approach for estimating the mean in the context of gradient-based sampling techniques.  Our overall procedure is a robust variant of the Unadjusted Langevin Algorithm (ULA) that we refer to as ``Rob-ULA.'' The underlying ULA algorithm and its variants have been used for efficient large-scale Bayesian posterior sampling~\citep{roberts1996,welling2011} and their convergence analysis has been a recent topic of interest~\citep{dalalyan2017,durmus2017,Xiang_overdamped,dwivedi2018log,MCMC_nonconvex}; see Section~\ref{sec:rel_samp} for a detailed overview. Informally, our main result shows that after $T= \Ot(d/\acc)$ iterations of \alg, the iterate $\param_T$ has a distribution $p_T$ such that $\text{dist}(p_T, p^*) \leq \acc + \Ot(\epsilon)$, where $\epsilon$ is the fraction of corrupted points in the data set.


The remainder of the paper is organized as follows: Section~\ref{sec:relw} contains a discussion of the related literature, Section~\ref{sec:background} discusses relevant background as well as the formal problem setup, and Section~\ref{sec:alg} describes the proposed algorithm \alg and states our main theorem regarding its convergence.  In Section~\ref{sec:cons} we discuss the fundamental problems of Bayesian mean estimation and linear regression in the robust setting. Section~\ref{sec:exps} consists of experimental evaluation of \alg on synthetic as well as real world data sets and we conclude with Section~\ref{sec:conc}.

\section{Related Work}\label{sec:relw}
In this section, we review related work on robust estimation procedures in both the Bayesian and frequentist settings.  We also discuss work on using Langevin dynamics to sample from distributions over continuous spaces.

\subsection{Robust statistical procedures}
There are many threads in the literature on robust estimation and outlier detection~\citep{huber1973, box1953, de1961}. In the frequentist parameter estimation setting, the most commonly studied model is Huber's classical $\epsilon$-contamination model. There has also been a recent focus on an adversarial paradigm that is devoted to developing computationally efficient problem-dependent estimators for mean estimation ~\citep{lai2016, diakonikolas2016}, linear regression~\citep{klivans2018, bhatia2015, bhatia2017, suggala2019}, and general risk minimization problems~\citep{prasad2018, diakonikolas2018}. Particular relevant to our setup are \cite{prasad2018} and \cite{diakonikolas2018} which utilize the robust mean estimators of ~\citep{lai2016, diakonikolas2016} to robustify gradient-based procedures for empirical risk minimization.

The study of robustness in the Bayesian framework has  focused primarily on developing models and priors that have minimal impact on inference. An important line of work has focused on the sensitivity of the posterior distribution to and has led to the development of noninformative priors~\cite{berger1994, minsker2017, miller2018}. These methods are orthogonal to those considered in the current paper, as they do not aim to robustify inferential procedures against corruptions in the observed data set. In principle a complete Bayesian model would have the capacity for explaining the outliers present in the data, but this would require performing an integral over all models with a proper prior. Such an approach would generally be computationally intractable.

An important class of procedures for handling outliers in the data set focuses on \emph{reweighing} the data points via a transformation of the likelihood function. Huber \cite{huber2011} considers assigning binary weights to data points and identifies model-dependent procedures to identify outliers. In contrast, Wang et al. \cite{wang2017} consider these weights as latent variables and infers these weight variables along with the parameters of the model. These methods are susceptible to the choice of priors over these weighting variables. An alternate set of robust procedures are based on the idea of \emph{localization}~\cite{de1961, wang2018}: each data point is allowed to depend on its own version of the latent variable and these individual latent variables are tied together with the use of a hyperparameter, often fitted using empirical Bayes estimation. Although these methods build on intuitive notions of outlier removal, there is little theoretical understanding of the kind of outliers that these methods can tolerate.

\subsection{Sampling methods}\label{sec:rel_samp}
Turning to sampling methods for posterior inference, there have been various zeroth-order~\cite{ lovasz1993, mengersen1996} and first-order methods~\cite{ ermak1975, roberts1998, neal2011} proposed for sampling from distributions over continuous spaces. Our focus in this paper is on the overdamped Langevin MCMC method which was first proposed by Ermak~\cite{ermak1975} in the context of molecular dynamics applications. Its nonasymptotic convergence (in total variation distance) was first studied by Dalalyan~\cite{dalalyan2017} for log-smooth and log-strongly concave distributions. Cheng and Bartlett~\citep{Xiang_overdamped} extended the analysis to obtain a similar convergence result in Kullback-Leibler divergence. 
Such nonasymptotic convergence guarantees are essential to understanding the robustness of computational procedures as they simultaneously capture the dependence of the sampling error on the number of iterations $T$, the dimensionality $d$, and the contamination level~$\epsilon$.
\section{Background}\label{sec:background}
In the section we briefly review relevant background on Bayesian computation and we formally describe our problem setup. 

\subsection{Bayesian modelling}\label{sec:back_bm}
Given parameters $\param \in \R^d$ and a data set $\data = \{\z_1, \z_2, \ldots, \z_n\}$, we assume that the statistician has specified a prior distribution, $\prior(\param|\hyper)$, and a likelihood, $p(\z|\param)$.  We can then form the posterior distribution, $p(\param|\data, \hyper)$, as follows:
\begin{equation*}\label{eq:post_update}
p(\param|\data, \hyper) \propto \prior(\param|\hyper)\cdot\prod_{i=1}^np(\z_i|\param)\;.
\end{equation*}
We are generally concerned with the estimation of some test function $\tfunc(\param)$ under the posterior distribution, which is accomplished in the Bayesian setting by computing a posterior mean:
\begin{equation*}
  \pmean(\tfunc|\data, \hyper) := \int_{\R^d} \tfunc(\param)p(\param|\data, \hyper) \rd\param\;.
\end{equation*}
In practice, one way of computing this posterior mean is to use a Monte Carlo algorithm to generate a sequence of samples $\{\param_t\}_{t=1}^T$ and form an estimate $\hat{\pmean}(\tfunc|\data, \hyper)$:
\begin{equation*}
  \hat{\pmean}(\tfunc|\data, \hyper) = \frac{1}{T}\sum_{t = 1}^T\tfunc(\param_t)\;.
\end{equation*}

\subsection{Unadjusted Langevin Algorithm}\label{sec:back_ula}
We consider a specific Monte Carlo algorithm, the Unadjusted Langevin Algorithm (ULA), for sampling from probability distributions over continuous spaces. Generically, we consider distributions over $\R^d$ of the form
\begin{equation*}
  p^*(\param) \propto \exp(-\f(\param)),
\end{equation*}
for a class of functions $\f$ which are square integrable. The ULA algorithm starts from some initial point $\param_0 \in \R^d$ and defines a sequence of parameters, $\{\param_{k, \stp}\}_{k=1}^T$, according to the following update equation:
\begin{equation}\label{eq:ula_update}
\param_{k+1, \stp} = \param_{k, \stp} -\stp_{k+1}\grad\f(\param_{k, \stp}) + \sqrt{2\stp_{k+1}}\gnoise_{k+1},
\end{equation}
where $\stp = \{\stp_i > 0\}$ denotes a sequence of step sizes and $\{\gnoise_i\}_{i \in \N} \sim \mathcal{N}(0, \text{I}_{d\times d})$ are i.i.d. Gaussian vectors. The Markov chain in Equation~\eqref{eq:ula_update} is the Euler discretization of a continuous-time diffusion process $\{\lang_t\}_{t\geq 0}$ known as the Langevin diffusion. The stochastic differential equation governing the Langevin diffusion is given by
\begin{equation}\label{eq:lang}
\rd\lang_t = -\grad f(\lang_t)\rd t + \sqrt{2}\rd B_t\;, \quad t\geq 0,
\end{equation}
where $\{B_t\}_{t\geq 0}$ represents a $d$-dimensional Brownian motion. Denoting the distribution of $\param_{k, \stp}$ by $p_{k, \stp}$, Cheng and Bartlett \citep{Xiang_overdamped} showed that $\kldiv{p_{k, \stp}}{p^*} \leq \epsilon$  after $t = \Ot(\frac{d}{\epsilon})$ steps for functions $\f$ which are smooth and strongly-convex.  Specializing to the Bayesian modelling setup we rewrite the posterior distribution as:
\begin{equation*}
  p(\param|\data, \hyper) \propto \exp\left(\log(\prior(\param|\hyper)) - \sum_{i=1}^n \g_i(\param) \right)\;,
\end{equation*}
where $\g_i(\param) := -\log(p(\z_i|\param))$ is the negative log-likelihood corresponding to the $i^{th}$ data point. The ULA algorithm can then be used to form an approximation to the posterior as follows: 
\begin{equation*}
  \param_{k+1, \stp} = \param_{k, \stp} -\stp_{k+1}\left(-\grad\log(\prior(\param_{k, \stp}|\hyper) + \sum_{i=1}^n\grad \g_i(\param_{k, \stp}) \right) + \sqrt{2\stp_{k+1}}\gnoise_{t+1},
\end{equation*}
where $\stp$ and $\gnoise_{t+1}$ are the step-size sequence and independent Gaussian noise respectively.

\subsection{Problem Setup}\label{sec:prob_setup}
We turn to a formal treatment of the robustness problem in the Bayesian setting. We consider the $\epsilon$-contamination model introduced by Huber \cite{huber1964} and let the collection of $n$ data points be obtained from the following mixture distribution:
\begin{equation}\label{eq:model}
  \z_i \sim (1-\epsilon)\Pc + \epsilon \Pa\,,
\end{equation}
where $\Pc$ denotes the true underlying generative distribution while $\Pa$ is any arbitrary distribution. A data set $\data$ drawn from such a mixture distribution has each data point $\z_i$  corrupted adversarially with probability $\epsilon$. We denote by $\data_c$ the subset of data points in $\data$ sampled from the true distribution $\Pc$ and similarly let $\data_a$ denote the subset of data sampled from $\Pa$. Given data points $\data = \data_c \cup \data_a$, the likelihood function $p(\z|\param)$ and the prior $p_\param(\param|\hyper)$, we aim to form a \emph{clean} posterior distribution given by:
\begin{equation*}
  p(\param|\data_c, \hyper) \propto \prior(\param|\hyper)\cdot\prod_{i\in \data_c}(\z_i|\param)\;.
\end{equation*}
Accordingly, as in Section~\ref{sec:back_bm}, we would like to \emph{robustly} estimate the mean of the test function $\tfunc(\param)$ under the uncorrupted posterior $p(\param|\data_c, \hyper)$:
\begin{equation*}
  \pmean(\tfunc|\data, \hyper) := \int_{\R^d} \tfunc(\param)p(\param|\data_c, \hyper) \rd\param,
\end{equation*}
which we approximate via an estimate:
\begin{equation*}
  \hat{\pmean}(\tfunc|\data_c, \hyper) = \frac{1}{T}\sum_{t = 1}^T\tfunc(\param^c_t)\;.
\end{equation*}
In the following section we present an algorithm, \alg, for generating the sequence of samples $\{\param^c_t\}_{t=1}^T$ and provide theoretical guarantees on its convergence properties. The main idea is to exploit gradient-based iterative sampling methods, and to leverage a robust mean estimation procedure to compute a robust estimate of the gradient at each iteration.
\begin{algorithm}[t!]
	\DontPrintSemicolon
	\KwIn{Data set $\data$, step-size sequence $\stp$, initial covariance scaling $\covsc$, timesteps $T$, prior distribution $\prior(\param|\hyper)$, hyperparameters $\hyper$, likelihood~function~$p(\z|\param)$, corruption level $\epsilon$}
  Sample $\param_0 \sim \mathcal{N}(0, \covsc I_d)$\;
  \For{$k = 1, \ldots, T$}{
  Let $\g_i(\param_{k-1, \stp}) := -\log(p(\z_i|\param_{k-1, \stp}))$ for $i = 1\ldots n$\;
  $\widehat{\nabla}U_{\param} = \text{RobustGradientEstimate}(\{\nabla\g_i(\param_{k-1, \stp})\}_{i=1}^n, \epsilon, d)$\;
  $\param_{k, \stp} = \param_{k-1, \stp} - \stp_{k}(n \cdot\widehat{\nabla}U_{\param} - \grad\log(\prior(\param_{k-1, \stp}|\hyper)) + \sqrt{2\stp_{k}}\gnoise_{k}, \quad \text{where } \gnoise_{k} \sim~\mathcal{N}(0, I_d)$\;
  }
  \KwOut{Iterates $\{\vtheta_k\}$}
	\caption{\alg: Robust Unadjusted Langevin Algorithm}
	\label{alg:rula}
\end{algorithm}

\section{\alg: Robust Unadjusted Langevin Algorithm}\label{sec:alg}
We turn to our proposed algorithm, \alg (Algorithm~\ref{alg:rula}), which aims to solve the robust posterior inference problem defined in Section~\ref{sec:prob_setup}. \alg is a simple modification of the ULA algorithm, described in Section~\ref{sec:back_ula}, where in each iteration instead of using the complete set of data points for computing the gradient, we construct a \emph{robust} estimator of the gradient and update the parameter using this estimate. This robust estimator ensures that the outlier data points do not exert too much influence on the gradient and allow \alg to obtain samples from a distribution close to the clean posterior distribution:
\begin{equation}\label{eq:clean_post}
  p(\param|\data_c, \hyper) \propto \prior(\param|\hyper)\cdot\prod_{i\in \data_c}(\z_i|\param)\;.
\end{equation}
Before proceeding to establish the convergence guarantees for \alg, we present the robust gradient estimation procedure.

\subsection{Robust Gradient Estimation}\label{sec:rob_grad}
\begin{algorithm}[t!]
	\DontPrintSemicolon
	\KwIn{Sample Gradients $S = \{\nabla \g_i(\param) \}_{i=1}^n$, Corruption Level $\epsilon$, Dimension $d$}
  $\tilde{S}$ = Outlier Truncation($S, \epsilon, d$)\;
  \If{$d=1$}{
  $\hat{\mu}\leftarrow$  mean($\tilde{S}$)
  }\Else{
  $\Sigma_{\tilde{S}} \leftarrow$ sample covariance of $\tilde{S}$\;
  Let $V=$ span of top $d/2$ principal components of $\Sigma_{\tilde{S}}$ and $W = V^c$\;
  $S_1 \leftarrow P_V(\tilde{S})$ where $P_V$ is projection onto $V$\;
  $\hat{\mu}_V \leftarrow $ Robust Gradient Estimator $(S_i, \epsilon, d/2)$\;
  $\hat{\mu}_W \leftarrow \text{mean}(P_W(\tilde{S}))$\;
  $\hat{\mu} \leftarrow \hat{\mu}_V + \hat{\mu}_W$
  }
  \KwOut{Estimate of Robust Gradient: $\hat{\mu}$}
	\caption{RobustGradientEstimate}
	\label{alg:rge}
\end{algorithm}

\begin{algorithm}[t!]
	\DontPrintSemicolon
	\KwIn{Sample Gradients $S = \{\grad\g_i(\param) \}_{i=1}^n$, Corruption Level $\epsilon$, Dimension $d$}
  \If{d = 1}{
  $[a,b] \leftarrow $ smallest interval containing $\left(1-\epsilon\right)^2$ fraction of points.\;
  $\tilde{S} \leftarrow S\cap [a,b]$\;
  }
  \Else{
  Let $[S]_i$ be samples with only $i^{th}$ coordinate\;
  \For{$i = 1, \ldots, d$}{
     $a[i]$ = Robust Gradient Estimator$([S]_i, \epsilon, 1)$\;
  }
  Let $B(r, a)$ be ball of smallest radius centred at $a$ containing $(1-\epsilon)^2$ fraction of points.\;
 $\tilde{S} \leftarrow S\cap B(r,a)$\;
  }
  \KwOut{Points after outlier removal : $\tilde{S}$}
	\caption{Outlier Truncation}
	\label{alg:ot}
\end{algorithm}

Algorithm~\ref{alg:rge} describes our robust gradient estimation procedure.  Based on the robust mean estimator of Lai et al.~\cite{lai2016}, it takes as input the gradients of the negative log-likelihoods $\grad\g_i(\param)$ and outputs an estimate of the \emph{robust mean} of the gradient vectors ($\widehat{\grad}U_\param$ in Algorithm~\ref{alg:rula}), assuming a fraction $\epsilon$ of them are arbitrarily corrupted. Algorithm~\ref{alg:rula} then scales this gradient estimate by the number of samples $n$, to obtain a robust estimate of gradients of the likelihood $\sum_{i=1}^n \grad\g_i(\param)$.

Note that the model described in Section~\ref{sec:prob_setup} assumes that each data point $\z$ is sampled i.i.d.\ from the mixture distribution $(1-\epsilon)P + \epsilon Q$, where $P$ represents the true generative distribution and $Q$ can be any arbitrary distribution. An application of the Hoeffding bound for Bernoulli random variables shows that with probability at least $1-\delta$, the fraction of corrupted points $\epsilon_n$ in the sampled data set $\data$ satisfy
\begin{equation}\label{eq:bnd_ep_m}
  \epsilon - \sqrt{\frac{2}{n}\log\left(\frac{1}{\delta}\right)} \leq \epsilon_n \leq \epsilon + \underbrace{\sqrt{\frac{2}{n}\log\left(\frac{1}{\delta}\right)}}_{e_n}.
\end{equation}
For the remainder of the paper, we condition on this high probability event and state our results assuming this event holds. Following the proof strategy of \cite{lai2016} and \cite{prasad2018}, we derive a bound on the estimation error of the true average log-likelihood gradient,
\begin{equation*}
  \left\Vert\widehat{\grad}U_\theta - {\frac{1}{|\data_c |}\sum_{i \in \data_c} \grad\g_i(\param)} \right\Vert_2,
\end{equation*}
\emph{uniformly} for any value of the iterate $\param$ in the following lemma. We let $\grad U_\param := \frac{1}{|\data_c |}\sum_{i \in \data_c} \grad\g_i(\param)$ denote the true value of this average log-likelihood gradient.

\begin{lemma}[Robust Gradient Estimation]\label{lem:rob_bound}
  Let $P$ denote the uniform distribution over $\data_c$ and let $P_\param$ denote the corresponding distribution over $\grad \g_i(\param)$ with mean given by $\grad U_\param$, covariance $\Sigma_\param$ and fourth moment given by $C_4$. There exists a positive constant $C_1 > 0$ for which the robust mean estimator when instantiated with the contamination level $\gamma := \epsilon + e_n$, returns, with probability $1-\delta$, an estimate $\widehat{\grad}U_\param$ such that for all $\param \in \real^d$, we have that,
  \begin{equation*}
    \|\widehat{\grad}U_\param - \grad U_\param \|_2 \leq C_1C_4^{\frac{1}{4}}\sqrt{\gamma\log(d)\|\Sigma_{\theta}\|_2}\;.
  \end{equation*}
\end{lemma}
\paragraph{Remark.} Note that Proposition 1 of Prasad et al. \cite{prasad2018} presents a high-probability bound similar to ours which is applicable for a \emph{fixed} parameter $\theta$. Such a bound, however, does not suffice to ensure convergence of \alg because the additive Gaussian noise at every iterate requires us to obtain a \emph{uniform} high-probability recovery error bound (see Section~\ref{sec:conv} for details).  Lemma~\ref{lem:rob_bound} establishes a uniform bound for the specific distribution $P$ which is uniform over the clean data $\data_c$. This restriction of the distribution $P$ also allows us to avoid sample-splitting at every iteration of Algorithm~\ref{alg:rula} which was essential for both \cite{lai2016} and \cite{prasad2018}.

In addition, Lemma~\ref{lem:rob_bound} indicates that irrespective of the sample size $n$, one can estimate the mean of the gradient robustly up to error $\Ot(\sqrt{\epsilon \| \Sigma_\param\|_2})$. This implies that at each iteration of \alg, we incur an error of $\Ot(n\sqrt{\epsilon \| \Sigma_\param\|_2})$ since we scale the average gradient estimate by $n$ during the update. In Theorem~\ref{thm:RULA_conv} we show how with an appropriate choice of step size $\stp = \mathcal{O}(1/n)$, one can control the propagation of this bias in the convergence analysis for \alg. 

The detailed proof of Lemma~\ref{lem:rob_bound} can be found in Appendix~\ref{app:robust}.

\subsection{Convergence Analysis}\label{sec:conv}
In this section, we study the convergence of the proposed algorithm \alg. For ease of notation, we let $\f(\param;\data) = \sum_{i\in \data} \g_i(\param) - \log(p_\param(\param|\hyper))$ and similarly denote the clean and corrupted versions of the function $f(\param;\data_c)$ and $f(\param;\data_a)$. The objective of the robust Bayesian posterior inference problem is then to obtain samples from the clean posterior distribution given by $p^*(\param|\data, \hyper) \propto \exp\left(-\f(\param;\data_c) \right)$. For clarity of exposition, we drop the dependence of the posterior distribution on the data set $\data$ as well as the hyperparameters $\alpha$ and let $\f(\param) := \f(\param;\data_c)$.

We quantify the convergence of distribution $\density$ following a stochastic process to the stationary distribution $\target$ through the Kullback-Leibler divergence, $\kldiv{\density}{\target}$:
\begin{equation*}
\kldiv{\density}{\target}   = \int \density \ln \parenth{\frac{\density}{\target}} \rd \param \;.
\end{equation*}
We define the Wasserstein distance $\wass^2(p, q)$ between a pair of distributions $(p, q)$ as:
\begin{equation*}
    \wass^2(p, q) := \inf_{\zeta \in \Gamma(p, q)} \int \|x-y \|_2^2\;\rd\zeta(x, y)\,,
\end{equation*}
where $\Gamma(p, q)$ denotes the set of joint distributions such that the first set of coordinates has marginal $p$ and the second set has marginal $q$.

We begin by making the following assumptions on the function $\f(\param)$:
\begin{assumption}[Lipschitz smoothness]\label{A1}
  The function $\f(\param)$ is $\lip$-Lipschitz smooth and its Hessian exists for all $ \param\in\R^d$. That is,
  \begin{equation*}
    \lrn{\nabla \f(\param) - \nabla \f(\paramb)} \leq \lip \lrn{\param-\paramb}, \; \forall \param, \paramb \in\R^d \quad \text{and} \quad   \nabla^2 \f(\param) \text{ exists for all } \param\in\R^d\;.
    \end{equation*}
\end{assumption}

\begin{assumption}[Strong convexity]\label{A2}
The function $\f(\param)$ is $\conv$-strongly convex for all $\param \in \real^d$. That is,
\begin{equation*}
   \conv I \preceq \grad^2 \f(\param), \quad \forall \param \in\R^d.
\end{equation*}
We further denote the condition number of the function $\f$ as $\cond = \lip/\conv$.
\end{assumption}
The assumptions of Lipschitz smoothness and strong convexity are standard in both the sampling and optimization literatures. In addition to the assumptions, we define the average Lipschitz constant $\lipav  = \lip/n$ and the average strong convexity of $\f$ as $\convav = \conv/n$. We now state our main theorem concerning the convergence guarantees for \alg.
\begin{theorem}[Main Result]
\label{thm:RULA_conv}
Let $\p^*(\param) \propto \exp(- \f(\param))$, where $f$ satisfies Assumptions~\mbox{\ref{A1} and \ref{A2}}. Further, assume that the gradient estimates $\widehat{\grad}\f(\param)$ satisfy
\begin{equation*}
 \lrn{\nabla \f(\param_{k, \stp}) - \widehat{\nabla} {\f}(\param_{k, \stp})}^2 \leq n^2\epsilon \Cr \lrn{\Sigma_\param}_2 \log d \quad \text{and} \quad \lrn{\Sigma_\param}_2 \leq \Csa\lrn{\param - \pari}^2 + \Csb\;,
\end{equation*}
where $\Sigma_\param$ is the covariance of uniform distribution on $\grad g_i (\param)$ induced by the clean data set~$\data_c$, $\pari$ satisfies $\grad F(\pari) = 0$ and $\epsilon$ is the fraction of corrupted points, satisfying $\epsilon \leq \dfrac{\bar{m}^2}{4\Cr\Csa\log d}$. Then the iterates of \alg, when initialized with $\param_0\sim\NORMAL\parenth{0,\frac{1}{\lip}I_d}$ (with corresponding density $p_0$) and step size $\stp \leq \frac{1}{n\lipav}$ (and define $h:= n\stp\leq\frac{1}{\lipav}$), satisfy:
\small{\emph{
\begin{align*}
\wass^2(\p_{k\stp}, p^*)
&\leq
\dfrac{2e^{\displaystyle - \bar{m} k h}}{n\bar{m}} \kldiv{p_0}{p^*}
+ 8\left( \Cr\Csb\dfrac{\bar{L}^4}{\bar{m}^4} \epsilon \log d + \dfrac{\bar{L}^4}{\bar{m}^3} \dfrac{d}{n} \right) h^2
+ 4\dfrac{\bar{L}^2}{\bar{m}^2} \dfrac{d}{n} h\\
&\quad +
\epsilon\left(\dfrac{4\Cr\Csb}{\bar{m}^2} \log d + \dfrac{8\Cr\Csa}{\convav^2}\dfrac{d\log d}{n}\right)\\
&\leq
\dfrac{1}{\bar{m}} \log{\dfrac{\bar{L}}{\bar{m}}} \dfrac{d}{n} e^{\displaystyle - \bar{m} k h}
+ 8\left( \Cr\Csb\dfrac{\bar{L}^4}{\bar{m}^4} \epsilon \log d + \dfrac{\bar{L}^4}{\bar{m}^3} \dfrac{d}{n} \right) h^2
+ 4\dfrac{\bar{L}^2}{\bar{m}^2} \dfrac{d}{n} h\\
&\quad +
\epsilon\left(\dfrac{4\Cr\Csb}{\bar{m}^2} \log d + \dfrac{8\Cr\Csa}{\convav^2}\dfrac{d\log d}{n}\right)\;,
\end{align*}}}
where $\p_{k\stp}$ represents the distribution of the iterate $\param_{k, \stp}$.
\end{theorem}

\paragraph{Remarks.} Before proceeding to the proof of this theorem, a few comments are in order. First observe that the error term consists of three different components:
\begin{equation*}
  \underbrace{\dfrac{2e^{\displaystyle - n \bar{m} k\eta}}{n\bar{m}} \kldiv{p_0}{p^*}}_{(I)}
  + \underbrace{C\left(  \dfrac{\bar{L}^4}{\bar{m}^4} \epsilon \log d + \dfrac{\bar{L}^4}{\bar{m}^3} \dfrac{d}{n} \right) h^2
  + 4\dfrac{\bar{L}^2}{\bar{m}^2} \dfrac{d}{n} h}_{(II)}
  +
  \underbrace{\dfrac{C}{\bar{m}^2} \epsilon \log d}_{(III)},
\end{equation*}
where a) term (I) comprises an exponentially decaying dependence (with the number of time-steps $t$) on the initial error $\kldiv{p_0}{\target}$, b) term (II) is a discretization error term and c) term (III) captures the dependence on the fraction of corrupted points $\epsilon$ and vanishes as $\epsilon$ goes to zero.

For any given accuracy $\acc$, if the step size and the number of iterations satisfy:
\begin{equation*}
  \stp = \mathcal{O}\left(\frac{\acc}{n\kappa \lipav d}\right) \quad \text{and} \quad T \geq \mathcal{O}\left(\frac{\lipav}{\bar{m}}\log\left(\frac{\kldiv{p_0}{p^*}}{n\bar{m}\acc}\right)\right),
\end{equation*}
then the error in convergence can be bounded as
\begin{equation*}
  \wass^2(\p_{T, \stp}, p^*) \leq \acc + \Ot\left(\frac{\epsilon}{\convav^2} \right).
\end{equation*}
As we show in Section~\ref{sec:cons}, for problems such as Bayesian linear regression and Bayesian mean estimation, the average strong convexity parameter $\convav$ scales independently of the sample size $n$. This implies that the resulting error can be bounded by $\acc + \Ot(\epsilon)$. While the accuracy can be set to arbitrarily small values which would result in a corresponding increase in the number of time steps, there is a bias term depending on the contamination level $\mathcal{O}(\epsilon)$ which cannot be reduced by either increasing the sample size or by increasing the number of iterations. This is consistent with results in the frequentist literature \citep{bhatia2015,diakonikolas2016, lai2016, prasad2018}, which show that such inconsistency is a result of the adversarial corruptions and in general cannot be avoided.

Lemma~\ref{lemma:initial_dist} in Appendix~\ref{app:lmc} presents the following bound on the initial error:
\begin{equation*}
  {\kldiv{\p_0}{\p^*}} = \int \p_0(\vx)\log\left(\dfrac{\p_0(\vx)}{\p^*(\vx)}\right)\rd \vx \leq
\dfrac{d}{2}\log\dfrac{\lipav}{\convav}.
\end{equation*}

\subsection*{Proof of Theorem~\ref{thm:RULA_conv}}
We proceed to a proof of our main convergence result, Theorem~\ref{thm:RULA_conv}. We begin by considering the process described by \alg as a discretization of the Langevin dynamics given by Equation~\ref{eq:lang}, with the following gradient estimate:
\begin{align}
  \label{eq:disc_LD}
  \vTheta_{(k+1)\stp} = \vTheta_{k\stp} - \stp\widehat{\nabla}{\f}(\vTheta_{k\stp})  + \sqrt{2} (B_{(k+1)\eta} - B_{k\eta}),
\end{align}
where $\vTheta_{k\stp}$ represents the random variable describing the process at the $k^{th}$ iterate using step size $\stp$. This is equivalent to defining the following stochastic differential equation
\begin{align}
  \label{eq:disc_SDE_2}
  \rd \vTheta_t = - \widehat{\nabla}{\f}(\vTheta_{k\step}) \rd t + \sqrt{2} \rd B_t, \quad \text{ for } \quad {k\stp < t \leq (k+1)\stp}\;.
\end{align}
We next state a lemma which provides a bound on the variance of the distribution of the $k^{th}$ iterate.
\begin{lemma}\label{lem:var}
For $\Theta_t$ following Eq.~\eqref{eq:disc_SDE_2}, if $\Theta_0\sim\mathcal{N}\left(0,\dfrac{1}{n\lipav}\mI\right)$, $\epsilon \leq \dfrac{\convav^2}{4C\log d}$, and $h:= n \eta \leq\dfrac{1}{\bar{L}}$, then for all $k\in\mathbb{N}^+$,
\[
\E{\lrn{\vTheta_{k\stp} - \pari}_2^2} \leq \dfrac{4}{\bar{m}^2} C'\epsilon \log d + \dfrac{4d}{n \bar{m}},
\]
where $C$ is a universal constant and $\pari$ is the fixed point satisfying $\grad f(\pari) = 0$.
\end{lemma}
The proof of this lemma is deferred to Appendix~\ref{app:lmc}. Treating Lemma~\ref{lem:var} as given, we proceed to the proof of Theorem~\ref{thm:RULA_conv}.

We consider the dynamics in Equation~\ref{eq:disc_SDE_2} within the time range $k\stp < t \leq (k+1)\stp$. From the Girsanov theorem~\citep{SDE_book} we have that $\vTheta_t$ admits a density function $p_t$ with respect to the Lebesgue measure.
This density function can also be represented as
\begin{equation*}
p_t(\vtheta)=\int p_{k\stp}(\vzeta) p(\vtheta,t|\vzeta,k\stp) \rd \vzeta,
\end{equation*}
where $p(\vtheta,t|\vzeta,k\stp)$ is the weak solution to the following Kolmogorov forward equation:
\begin{equation*}
\dfrac{\partial p(\vtheta,t|\vzeta,k\stp)}{\partial t} = \nabla^T \big(\nabla p(\vtheta,t|\vzeta,k\stp) + \widehat{\nabla}{\f}(\vzeta) p(\vtheta,t|\vzeta,k\stp) \big),
\end{equation*}
where $p(\vtheta,t|\vzeta,k\stp)$ and its derivatives are defined via $P_t(f) = \int f(\vtheta) p(\vtheta,t|\vzeta,k\stp) \rd \vtheta$ as a functional over the space of smooth bounded functions on $\mathbb{R}^d$ (we refer the readers to ~\citep{Schilling_book} for more details). As shown by Cheng and Bartlett~\citep{Xiang_overdamped}, the time derivative of the KL divergence along $p_t$ is given by:
\begin{equation*}
\ddt \kldiv{\p_t}{p^*}
= - \E{\lin{ \nabla \ln\lrp{\frac{\p_t(\vTheta_t)}{\p^*(\vTheta_t)}}, \nabla \ln p_t(\vTheta_t) + \widehat{\nabla}{\f}(\vTheta_{k\stp}) }},
\end{equation*}
where the expectation is taken with respect to the joint distribution of $\vTheta_t$ and $\vTheta_{k\stp}$. Hence
\begin{align}\label{eq:kl_dt}
\ddt \kldiv{\p_t}{p^*}
&\stackrel{\1}{=}  - \E{\lin{ \nabla \ln\lrp{\frac{\p_t(\vTheta_t)}{\p^*(\vTheta_t)}}, \nabla \ln\lrp{\frac{\p_t(\vTheta_t)}{\p^*(\vTheta_t)}}
+ \left(\widehat{\nabla}{\f}(\vTheta_{k\stp}) - \nabla \f(\vTheta_t) \right) }} \nonumber\\
&=  - \E{ \lrn{ \nabla \ln\lrp{\frac{\p_t(\vTheta_t)}{\p^*(\vTheta_t)}} }^2}
+ \E{\lin{ \nabla \ln\lrp{\frac{\p_t(\vTheta_t)}{\p^*(\vTheta_t)}}, \nabla \f(\vTheta_t) - \widehat{\nabla}{\f}(\vTheta_{k\stp}) }}\;,
\end{align}
where $\1$ follows from the definition of $p^*(\param) \propto \exp(-f(\param))$. We first focus on the second term in the above expression which can be bounded as:
\begin{align}\label{eq:bnd_term2}
&\E{\lin{ \nabla \ln\lrp{\frac{\p_t(\vTheta_t)}{\p^*(\vTheta_t)}}, \nabla \f(\vTheta_t) - \widehat{\nabla}{\f}(\vTheta_{k\stp}) }}\nonumber \\
&\quad=  \E{ \lin{\nabla \ln\lrp{\frac{\p_t(\vTheta_t)}{\p^*(\vTheta_t)}}, \left(\nabla \f(\vTheta_t) - \nabla \f(\vTheta_{k\stp})\right) + \left(\nabla \f(\vTheta_{k\stp}) - \widehat{\nabla}{\f}(\vTheta_{k\stp})\right) }}
\nonumber\\
&\quad \stackrel{\1}{\leq}
\dfrac12 \E{ \lrn{\nabla \ln\lrp{\frac{\p_t(\vTheta_t)}{\p^*(\vTheta_t)}} }^2 }
+ \E{ \lrn{\nabla \f(\vTheta_t) - \nabla \f(\vTheta_{k\stp})}^2 }
+ \E{ \lrn{\nabla \f(\vTheta_{k\stp}) - \widehat{\nabla}{\f}(\vTheta_{k\stp})}^2 }
\nonumber\\
&\quad \stackrel{\2}{\leq}
\dfrac12 \E{ \lrn{\nabla \ln\lrp{\frac{\p_t(\vTheta_t)}{\p^*(\vTheta_t)}} }^2 }
+  L^2\E{ \lrn{ \vTheta_t - \vTheta_{k\stp} }^2 }
+  \Cr n^2\epsilon \left(\Csa \E{\lrn{\vTheta_{k\stp} - \pari}^2} + \Csb\right) \log d,
\end{align}
where $\1$ follows by an application of Young's inequality $2a^\top b \leq  \|a\|_2^2 + \| b\|_2^2$ and $\2$ follows from the point-wise assumption that $ \lrn{\nabla \f(\vtheta_{k, \stp}) - \widehat{\nabla}{\f}(\vtheta_{k, \stp})}^2 \leq n^2\epsilon \Cr \lrn{\Sigma_\vtheta}_2 \log d$ and that $\lrn{\Sigma_\vtheta}_2 \leq \Csa\lrn{\vtheta - \pari}^2 + \Csb$. Let us define new constant $C_{13} := \Cr \cdot\Csa$ and $C_{14} := \Cr \cdot\Csb$.

Next, we proceed to bound the term $\E{ \lrn{ \vTheta_t - \vTheta_{k\stp} }^2 }$ using Lemma~\ref{lem:var}. Let us define the variable $\tau:=t-k\stp\in(0,\stp]$ and bound the term as:
\begin{align}\label{eq:bnd_norm_diff}
\E{ \lrn{ \vTheta_t - \vTheta_{k\stp} }^2 }
 &\leq
\E{ \lrn{ - \nabla f(\vTheta_{k\stp}) \tau + \sqrt{2} (B_{(k+1)\stp} - B_{k\stp}) }^2 }\nonumber
\\ &{\leq}
\Ep{\vtheta\sim\p_{k\stp}}{ \lrn{ \nabla \f(\vtheta) }^2 } \tau^2 + 2 d \tau\nonumber
\\ &\stackrel{\1}{\leq}
\Ep{\vtheta\sim\p_{k\stp}}{ \lrn{ \vtheta - \pari}^2 } \bar{L}^2 \nu^2 + 2 \dfrac{d}{n} \nu\;,
\end{align}
where $\1$ follows from Assumption~\ref{A1} and we define $\nu=n\tau$ ($h=n\eta$). Plugging in the bounds obtained in Equations~\eqref{eq:bnd_term2} and~\eqref{eq:bnd_norm_diff} into Equation~\eqref{eq:kl_dt}, we get for \mbox{$k\stp < t \leq (k+1)\stp$}:
\begin{align*}
\ddt \kldiv{\p_t}{p^*} &\leq
- \dfrac12 \E{ \lrn{ \nabla \ln\lrp{\frac{\p_t(\vTheta_t)}{\p^*(\vTheta_t)}} }^2}
+ n^2 \bar{L}^4 \nu^2 \E{\lrn{\vTheta_{k\eta} - \pari}^2} + 2 nd \bar{L}^2\nu
\\ &\quad +\Cr n^2\epsilon \left( \Csa\E{\lrn{\vTheta_{k\eta} - \pari}^2} + \Csb \right) \log d
\\ &=
- \dfrac12 \Ep{\vtheta\sim\p_t}{ \lrn{ \nabla \ln\lrp{\frac{\p_t(\vtheta)}{\p^*(\vtheta)}} }^2}
+ n^2 \left(\bar{L}^4 \nu^2 + C_{13}\epsilon \log d\right) \E{\lrn{\vTheta_{k\eta} - \pari}^2}\\
&\quad +
n^2\epsilon C_{14} \log d + 2 nd \bar{L}^2  \nu
\\ &\stackrel{\1}{\leq}
- \conv \cdot \kldiv{\p_t}{p^*}
+ n^2 \left(\bar{L}^4 h^2 + C_{13} \epsilon \log d\right) \left( \dfrac{4}{\bar{m}^2} C_{14}\epsilon \log d + \dfrac{4d}{n \bar{m}} \right)
\\ &\quad +
n^2\epsilon C_{14} \log d + 2 n \bar{L}^2 d h
\\ &\stackrel{\2}{\leq}
- n\convav\kldiv{\p_t}{p^*}
+ E(h, n, d, \lipav, \convav),
\end{align*}
where $\1$ follows from an application of the log-Sobolev inequality with $\conv$ being the log-Sobolev constant and $\2$ follows from the fact that $\epsilon \leq \dfrac{\bar{m}^2}{4C_{13}\log d}$ and the substitution \begin{equation*}
E(h, n, d, \lipav, \convav)=\left( 4 C_{14} n^2 \dfrac{\bar{L}^4}{\bar{m}^2} \epsilon \log d + 4 n d \dfrac{\bar{L}^4}{\bar{m}}  \right) h^2 + 2 n d \bar{L}^2  h + 2C_{14} n^2 \epsilon \log d + 4C_{13}n\epsilon d \log d.
\end{equation*}
Finally, using Gronwall's inequality we have the following one-step progress equation:
\begin{equation*}
\kldiv{\p_{(k+1)\stp}}{p^*} - \frac{1}{n\convav} E(h, n, d, \lipav, \convav)
\leq
e^{\displaystyle - n\convav \stp} \left( \kldiv{\p_{k\stp}}{p^*} - \frac{1}{n\convav} E(h, n, d, \lipav, \convav) \right).
\end{equation*}
Repeated application of this progress inequality leads us to
\begin{align*}
\kldiv{\p_{k\stp}}{p^*}
\leq&
e^{\displaystyle -n\convav k\stp} \left( \kldiv{\p_{0}}{p^*} - \frac{1}{n\convav}  E(h, n, d, \lipav, \convav) \right)
+ \frac{1}{n\convav}  E(h, n, d, \lipav, \convav)
\\ \leq &
e^{\displaystyle - n\convav k\stp} \kldiv{\p_{0}}{p^*}
+ \frac{1}{n\convav} E(h, n, d, \lipav, \convav).
\end{align*}
The final result of the theorem can then be obtained by using Talagrand's inequality \citep{otto2000} which states that for the probability distributions $\p_{k\stp}$ and ${p^*}$, we have that
\begin{equation*}
  \wass^2(p_{k\stp}, p^*) \leq \frac{2}{n\convav}\kldiv{p_{k\stp}}{p^*} \leq \frac{2}{n\convav}e^{\displaystyle - n\convav k\stp} \kldiv{\p_{0}}{p^*}
  + \frac{2}{n^2\convav^2} E(h, n, d, \lipav, \convav),
\end{equation*}
which concludes the proof of the theorem. \hfill{\qed}

\section{Consequences for Mean Estimation and Regression}\label{sec:cons}
In this section, we study the fundamental problems of Bayesian mean estimation (Section~\ref{sec:rbme}) and Bayesian linear regression (Section~\ref{sec:rbls}) under the Huber $\epsilon$-contamination model.

\subsection{Robust Bayesian mean estimation}\label{sec:rbme}
We begin with the \emph{robust Bayesian mean estimation} (RBME) problem and instantiate the convergence guarantees for \alg (Algorithm \ref{alg:rula}) for this problem. For simplicity, we study the setup in which the likelihood is Gaussian:
\begin{equation*}\label{eq:model_bme}
  p(\z|\vtheta;\Sigma) =  \frac{1}{\sqrt{(2\pi)^{d}\det(\Sigma)}}\exp\left( -\frac{1}{2}\|\z- \vtheta\|_{\Sigma^{-1}}^2\right),
\end{equation*}
for a mean vector $\vtheta \in \mathbb{R}^d$ and a fixed positive definite covariance matrix $\Sigma \in \mathbb{R}^{d\times d}$. We consider the corresponding conjugate  prior over $\vtheta$ given by
\begin{equation*}\label{eq:prior_bme}
  \p(\vtheta;\vtheta_0, \Sigma_0) = \frac{1}{\sqrt{(2\pi)^{d}\det(\Sigma_0)}}\exp\left( -\frac{1}{2}\|\vtheta- \vtheta_0\|^2_{\Sigma_0^{-1}}\right),
\end{equation*}
where $\param_0$ is the mean and $\Sigma_0 \succ 0$ is the covariance matrix. The set of parameters $(\param_0, \Sigma_0)$ are the hyperparameters $\hyper$. Given data points $\data = \{\z_i\}_{i=1}^n$ sampled from the Huber contamination model, where $Z_i \stackrel{\text{i.i.d}}{\sim} (1-\epsilon)P + \epsilon Q$, with $Q$ being an arbitrary adversarially chosen distribution, the objective of the RBME problem is to sample from the posterior induced by the uncorrupted data points,
\begin{equation}\label{eq:post_rbme_m}
  p^*:= p(\vtheta|\D_c;\Sigma, \vtheta_0, \Sigma_0) \propto \exp\left( -\frac{1}{2}\|\vtheta- \vtheta_0\|^2_{\Sigma_0^{-1}}\right) \prod_{i \in \D_c} \exp\left( -\frac{1}{2}\|\z_i- \vtheta\|_{\Sigma^{-1}}^2\right),
\end{equation}
where $\D_c$ represents the subset of points in $\D$ sampled from the distribution $P$. We note that for data $X_i$ sampled from the Huber contamination model, we have from Equation \eqref{eq:bnd_ep_m} that with probability at least $1-\delta$,
\begin{equation}\label{eq:bnd_dc}
  n(1- \epsilon - e_n)\leq |\D_c| \leq n(1-\epsilon + e_n)\quad \text{where} \quad e_n := \sqrt{\frac{2}{n}\log\left(\frac{1}{\delta}\right)}.
\end{equation}
Let us denote by $\bar{\epsilon} := \epsilon + e_n$ and by $\underline{\epsilon} := \epsilon - e_n$ the corresponding upper and lower bounds on the number of corrupted data points. Following the notation in Section \ref{sec:alg}, the function $\f(\param)$ for this problem is given by:
\begin{equation}\label{eq:u_rbme}
  \f(\param) = \frac{1}{2}\left( \|\vtheta- \vtheta_0\|^2_{\Sigma_0^{-1}} + \sum_{i\in \D_c} \|\z_i- \vtheta\|_{\Sigma^{-1}}^2  \right).
\end{equation}
Also, we define the corresponding function $\g_i(\param)$ from Section~\ref{sec:background} as well as the sample mean and covariance for the clean data points.
\begin{equation*}
\g_i(\param) := \frac{1}{2} \|\z_i- \vtheta\|_{\Sigma^{-1}}^2, \quad  \mu_\z := \frac{1}{|\D_c|} \sum_{i\in \D_c}\z_i, \quad \Sigma_\z := \frac{1}{|\D_c|} \sum_{i\in \D_c}(\z_i - \mu_\z)(\z_i - \mu_\z)^\top.
\end{equation*}
The following corollary instantiates the guarantees of Theorem \ref{thm:RULA_conv} for the specific function $\f(\vtheta)$ defined for the RBME problem.
\begin{corollary}\label{cor:rbme}
  Consider the RBME problem with posterior given by Equation~\eqref{eq:post_rbme_m} and data sampled from the Huber model. Then the iterates of \alg with step size $\stp \leq \frac {1}{n\lipav}$ and $h:= n\stp$ satisfy:
  \small{\emph{
  \begin{equation*}
  \wass^2(\p_{k\stp}, p^*)
  \leq
  \dfrac{2e^{\displaystyle - n \bar{m} k\eta}}{n\bar{m}} \kldiv{p_0}{p^*}
  + C\left(\Csb\dfrac{\bar{L}^4}{\bar{m}^4} \epsilon \log d + \dfrac{\bar{L}^4}{\bar{m}^3} \dfrac{d}{n} \right) h^2
  + 4\dfrac{\bar{L}^2}{\bar{m}^2} \dfrac{d}{n} h +
  C\cdot \epsilon\dfrac{\Csb}{\bar{m}^2} \log d\;,
  \end{equation*}}}
  where $C$ is a constant depending on the fourth moment of the clean data $\data_c$ and
  \begin{equation*}
  \convav = \left(\frac{(1- \bar{\epsilon})}{\lambda_\text{max}(\Sigma)} + \frac{1}{n\lambda_\text{max}(\Sigma_0)}\right),\quad \lipav = \left(\frac{(1- \underline{\epsilon})}{\lambda_\text{min}(\Sigma)} + \frac{1}{n\lambda_\text{min}(\Sigma_0)}\right), \quad \text{and,} \quad \Csb = \frac{\lmx(\Sigma_z)}{\lmn(\Sigma)}\;,
  \end{equation*}
  with probability at least $1-\delta$ for $\epsilon < 1/2$.
\end{corollary}

\paragraph{Remark.} Observe that the step-size parameter $h$ is independent of $n$ while  both $\lipav$ and $\convav$ are asymptotically independent of the sample size $n$. The above bound shows that for an appropriately chosen step size one can obtain samples from a distribution which is $\Ot(\epsilon)$ away from the true distribution $p^*$. The number of iterations required to obtain such a sample scales linearly with the number of samples $n$, the dimension $d$ and the condition number $\kappa = \frac{\lipav}{\convav}$.

\subsection{Robust Bayesian linear regression}\label{sec:rbls}
We turn to the \emph{robust Bayesian Linear Regression} (RBLR) problem. For this problem, we let the data set $\D = \{ \z_i = (\x_i,y_i)\}_{i=1}^n$ be such that $\x_i \in \real^d$ and $y_i \in \real$ be the covariate vectors and response variables sampled from the Huber contamination model. Note that the Huber contamination model is on the variable $\z_i$ and hence allows for corruption in both the features $\x_i$ as well as the response variables $y_i$. In addition, we assume that there exists a vector $\vtheta^*$ such that
\begin{equation*}
  y_i = \ip{\x_i}{\vtheta^*} + z_i,
\end{equation*}
where $z_i \sim \mathcal{N}(0,\sigma^2)$ are sampled independently of $\x_i$. This assumption is for simplifying the presentation and in general one can work with $\param^*$ which is the best linear approximation to the data. For the RBLR problem, we consider likelihood functions of the form:
\begin{equation*}\label{eq:model_blr}
  p((\x,y)|\vtheta;\sigma) =  \frac{1}{\sqrt{2\pi\sigma^2}}\exp\left( -\frac{1}{2\sigma^2}(y - \ip{\x}{\vtheta})^2\right),
\end{equation*}
for a fixed variance parameter $\sigma^2$. Also, we consider a Gaussian prior over the parameter $\vtheta$,
\begin{equation*}\label{eq:prior_blr}
  \p(\vtheta;\vtheta_0, \Sigma_0) = \frac{1}{\sqrt{(2\pi)^{d}\det(\Sigma_0)}}\exp\left( -\frac{1}{2}\|\vtheta- \vtheta_0\|^2_{\Sigma_0^{-1}}\right),
\end{equation*}
for some fixed mean vector $\vtheta_0$ and positive-definite covariance matrix $\Sigma_0$ which form the set of hyperparameters $\hyper$. Given a data set $\D$ sampled from the Huber $\epsilon$-contamination model, the objective of the RBLR problem is to sample from the posterior induced by the uncorrupted set of data points,
\begin{equation}\label{eq:post_rblr}
  p(\vtheta|\D_c;\sigma, \vtheta_0, \Sigma_0) \propto \exp\left( -\frac{1}{2}\|\vtheta- \vtheta_0\|^2_{\Sigma_0^{-1}}\right) \prod_{i \in \D_c} \exp\left( -\frac{1}{2\sigma^2}(y_i - \ip{\x_i}{\vtheta})^2\right).
\end{equation}
Following a similar calculation to that in Section \ref{sec:rbme}, we have that with probability at least $1-\delta$,
\begin{equation}\label{eq:bnd_dc_lr}
  n(1- \eup)\leq |\D_c| \leq n(1-\elo).
\end{equation}
The corresponding function $\f(\vtheta)$ for the RBLR problem is then defined to be
\begin{equation}\label{eq:u_rblr_m}
\f(\vtheta) = \frac{1}{2}\left( \|\vtheta- \vtheta_0\|^2_{\Sigma_0^{-1}} + \frac{1}{\sigma^2}\sum_{i\in \D_c} (y_i - \ip{\x_i}{\vtheta})^2  \right).
\end{equation}
We denote by $\vtheta_\text{reg}$ the estimator which minimizes the function $U(\vtheta)$, and is given by:
\begin{equation*}
  \vtheta_{\text{reg}} := \left(\Sigma^{-1} +\frac{1}{\sigma^2}X_c^\top X_c\right)^{-1}\left(\frac{1}{\sigma^2}X_c^\top y_c + \Sigma^{-1}\vtheta_0 \right)\;,
\end{equation*}
where $X_c\in \real^{n_c\times d}$ represents the set of covariate vectors of the clean data points and $y_c \in \real^{n_c}$ represents the corresponding response values.  In addition, we define the following functions required for the analysis of the robust gradient estimator,
\begin{equation*}
  \g_i(\param) := \frac{1}{2\sigma^2}(y_i - \ip{\x_i}{\vtheta})^2 , \quad  \mu_x := \frac{1}{|\D_c|} \sum_{i\in \D_c}\x_i, \quad \wtS_x := \frac{1}{|\D_c|} \sum_{i\in \D_c}\x_i \x_i^\top.
\end{equation*}
We make the following moment assumptions on the covariates in the clean data set $\D_c$.

\begin{assumption}[Positive-definite data covariance]\label{A3} The unnormalized data covariance matrix is positive definite: $\wtS_x \succ 0$. \end{assumption}

\begin{assumption}[Bounded fourth moment]\label{A4} The data satisfies a bounded fourth moment condition, i.e., for every unit vector $v$, we have that
  \begin{equation*}\ex_{\x\sim_u\D_c}\left[(v^\top\x)^4\right] \leq C_{x,4}\left(\ex_{\x\sim_u\D_c}\left[(v^\top\x)^2\right]\right)^2,  \end{equation*}
    for come constant $C_{x,4}$.
\end{assumption}
Note that these assumptions are satisfied with high probability if say, each $\x_i \in \data_c$ is sampled i.i.d.\ from the standard normal distribution. The following corollary then instantiates the guarantees of Theorem \ref{thm:RULA_conv} for the specific function $\f(\vtheta)$ defined for the RBLR problem in Equation \eqref{eq:u_rblr_m}.

\begin{corollary}\label{cor:rbls}
Consider the RBLR problem described above with posterior given by Equation~\eqref{eq:post_rblr} and data sampled from the Huber model. Then the iterates of \alg with step size $\stp \leq  \frac{1}{n\lipav}$ and $h = n\stp$ satisfy:
\small{\emph{
\begin{align*}
\wass^2(\p_{k\stp}, p^*)
&\leq
\dfrac{2e^{\displaystyle - n \bar{m} k\eta}}{n\bar{m}} \kldiv{p_0}{p^*}
+ C\left(\Csb\dfrac{\bar{L}^4}{\bar{m}^4} \epsilon \log d + \dfrac{\bar{L}^4}{\bar{m}^3} \dfrac{d}{n} \right) h^2
+ 4\dfrac{\bar{L}^2}{\bar{m}^2} \dfrac{d}{n} h\\
&\quad +
C\epsilon\left(\dfrac{\Csb}{\bar{m}^2} \log d + \dfrac{\Csa}{\convav^2}\dfrac{d\log d}{n}\right)\;,
\end{align*}}}
with probability at least $1-\delta$ for $\epsilon \leq \frac{\convav^2}{4C\Csa\log d}$. The constant $C$ depends on $C_{x, 4}$ from Assumption~\ref{A4} and the remaining parameters are defined as:
\emph{
\begin{small}
\begin{gather*}
  \convav = (1- \bar{\epsilon}){\lambda_\text{min}(\wtS_x)} + \frac{1}{n\lambda_\text{max}(\Sigma_0)}, \quad \lipav = (1- \underline{\epsilon}){\lambda_\text{max}(\wtS_x)} + \frac{1}{n\lambda_\text{min}(\Sigma_0)},\quad
 \Csa = 2\sqrt{8C_{x,4}}\cdot\|\wtS_x\|_2\\
\Csb =\sqrt{C_{x,4}}\|\wtS_x \|_2 + 2\sqrt{8C_{x,4}}\cdot\|\wtS_x\|_2\cdot\|\vtheta^* - \vtheta_{\text{reg}}\|_2^2  +  \frac{(8C_{z,4})^{\frac{1}{2}}\cdot\sigma^2}{n^{\frac{1}{4}}}\log\left( \frac{e^2}{\delta}\right) + \sqrt{24}\sigma^2.
\end{gather*}
\end{small}}
\end{corollary}
\paragraph{Remark.} As in the RBME problem, the quantities $h$, $\lipav$ and $\convav$ are asymptotically independent of the sample size $n$. However, the guarantees above hold only for a value of $\epsilon \leq \Ot(\frac{1}{\kappa(\Sigma_\x)})$, that is, they depend on the condition number of the covariate distribution. Such a dependence seems inherent to the problem of linear regression since the adversary is allowed to corrupt the covariate vector arbitrarily.

\section{Experiments}\label{sec:exps}
In this section, we compare the performance of the proposed \alg with the non-robust variant ULA. We first compare them on synthetic data sets for the problem of Bayesian mean estimation and Bayesian linear regression in order to understand the variation in performance as a function of the problem parameters. In Section~\ref{sec:rwd}, we perform experiments comparing the algorithms on some real-world binary classification data sets using logistic regression. 

\subsection{Synthetic data sets}\label{sec:sd}
\subsubsection{Robust Bayesian Mean Estimation}
\begin{figure}[t!]
  \centering\hspace*{-4ex}
  \captionsetup{font=small}
\begin{tabular}{ccc}
  \includegraphics[width=.33\textwidth]{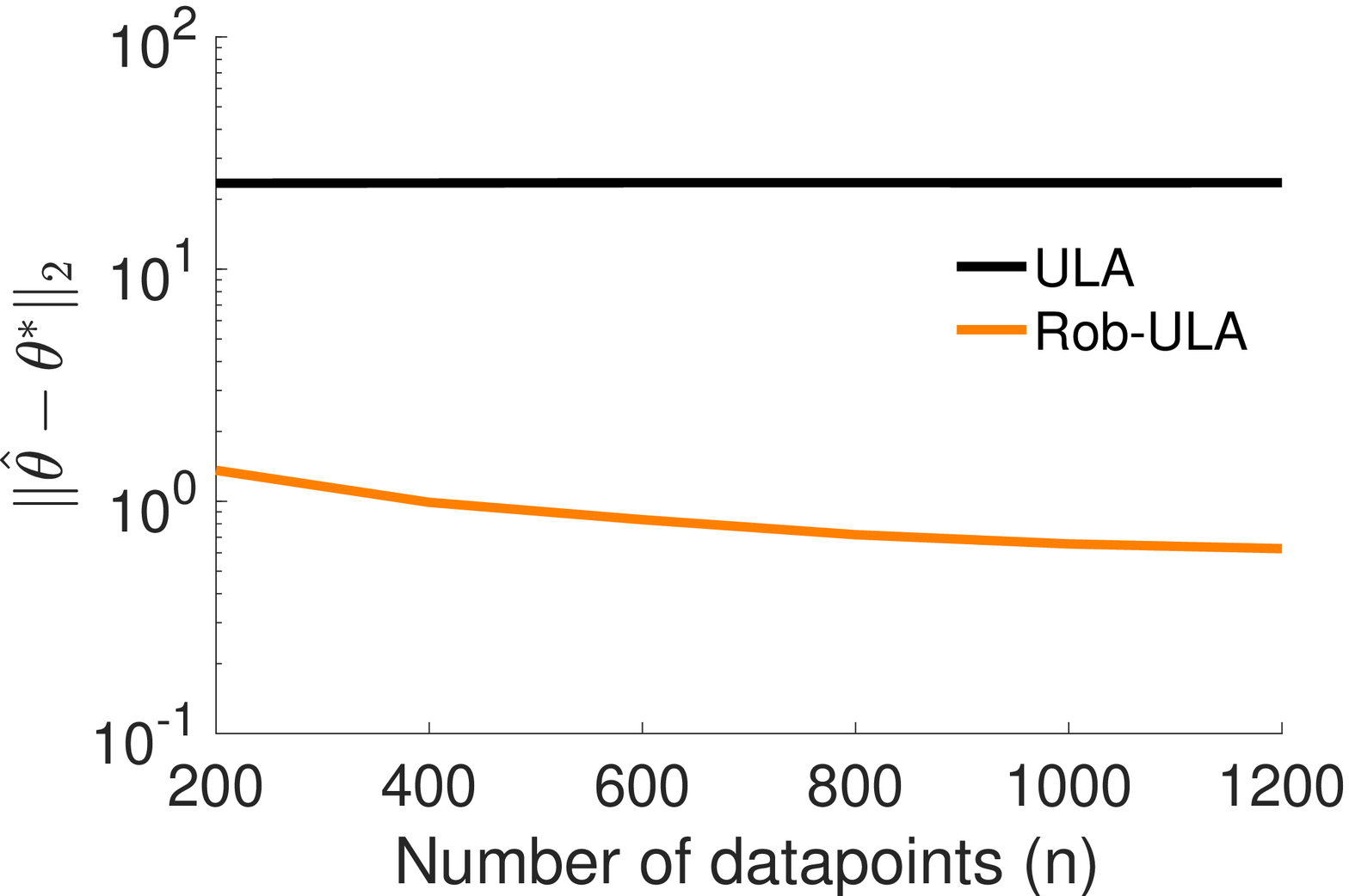}&
  \hspace{-4ex}
  \includegraphics[width=.33\textwidth]{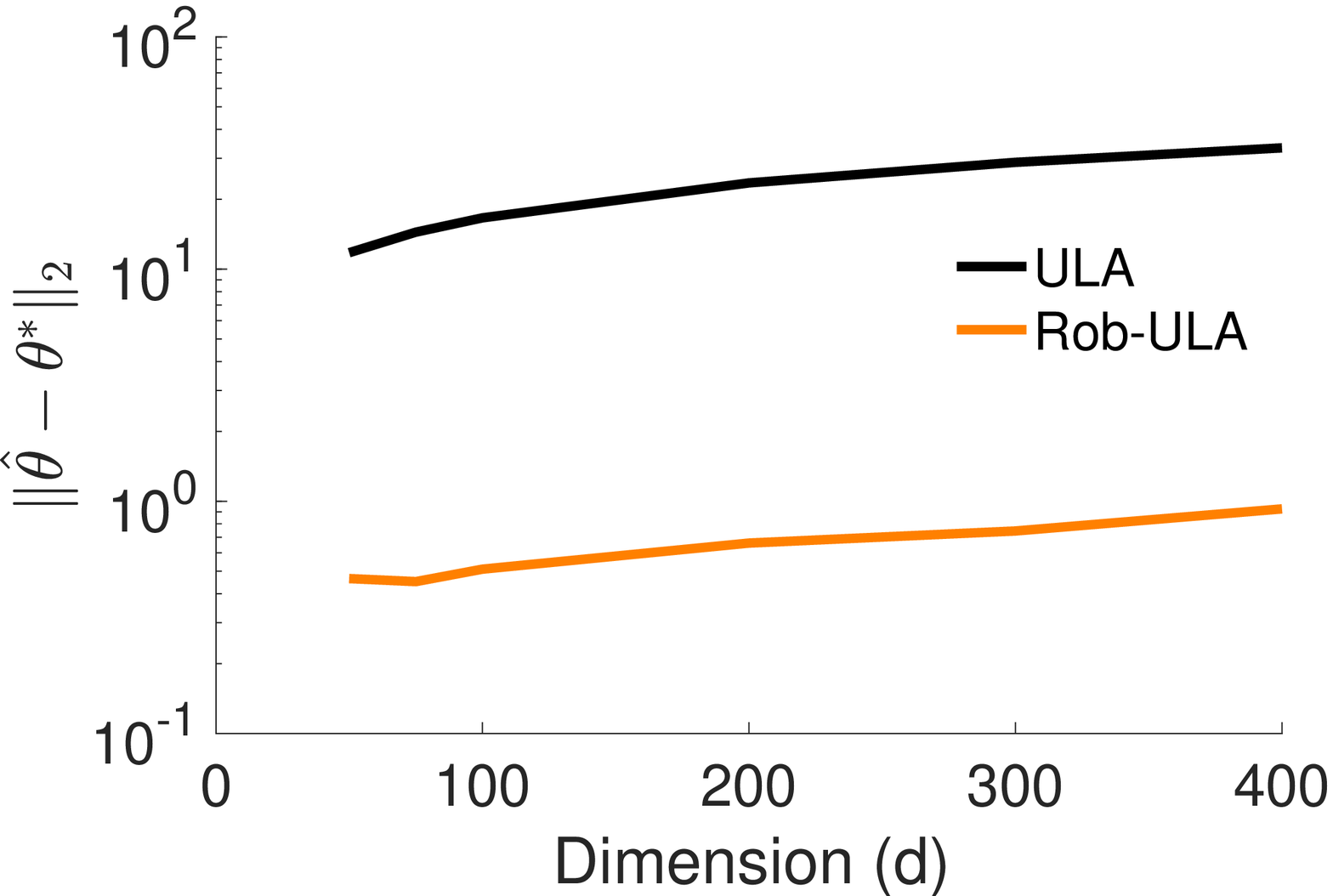}&
  \hspace{-4ex}
  \includegraphics[width=.33\textwidth]{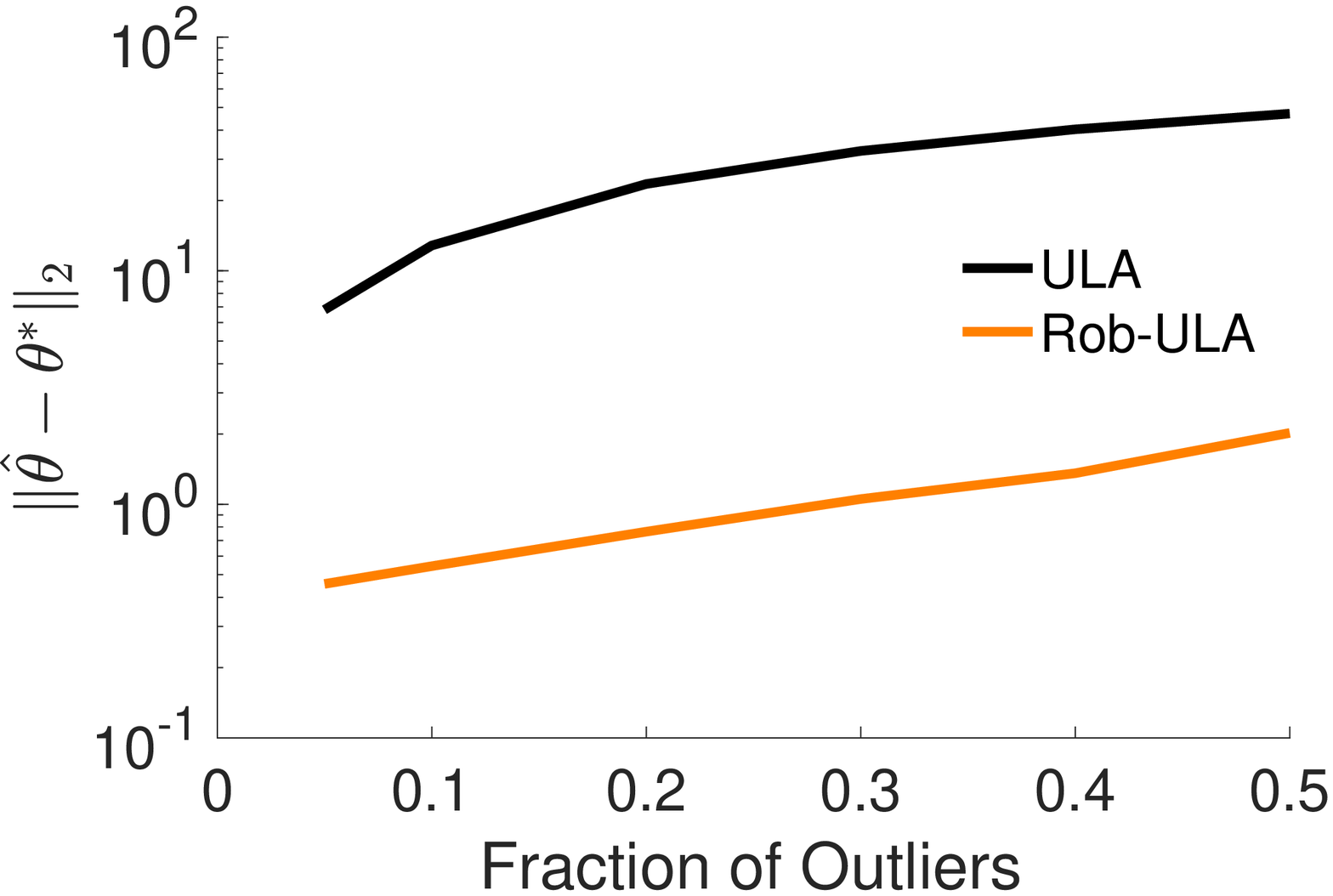}\\
(a) & (b) & (c)
\vspace*{-5pt}
\end{tabular}
\caption{\small{Robust Bayesian Mean Estimation (Parameter Estimation): \alg recovers the underlying parameter with smaller error as compared with the vanilla ULA. The recovery error increases with increasing dimension and fraction of outliers.}}\vspace*{-5pt}
  \label{fig:me}
\end{figure}

\begin{figure}[t!]
  \centering\hspace*{-4ex}
  \captionsetup{font=small}
\begin{tabular}{ccc}
  \includegraphics[width=.33\textwidth]{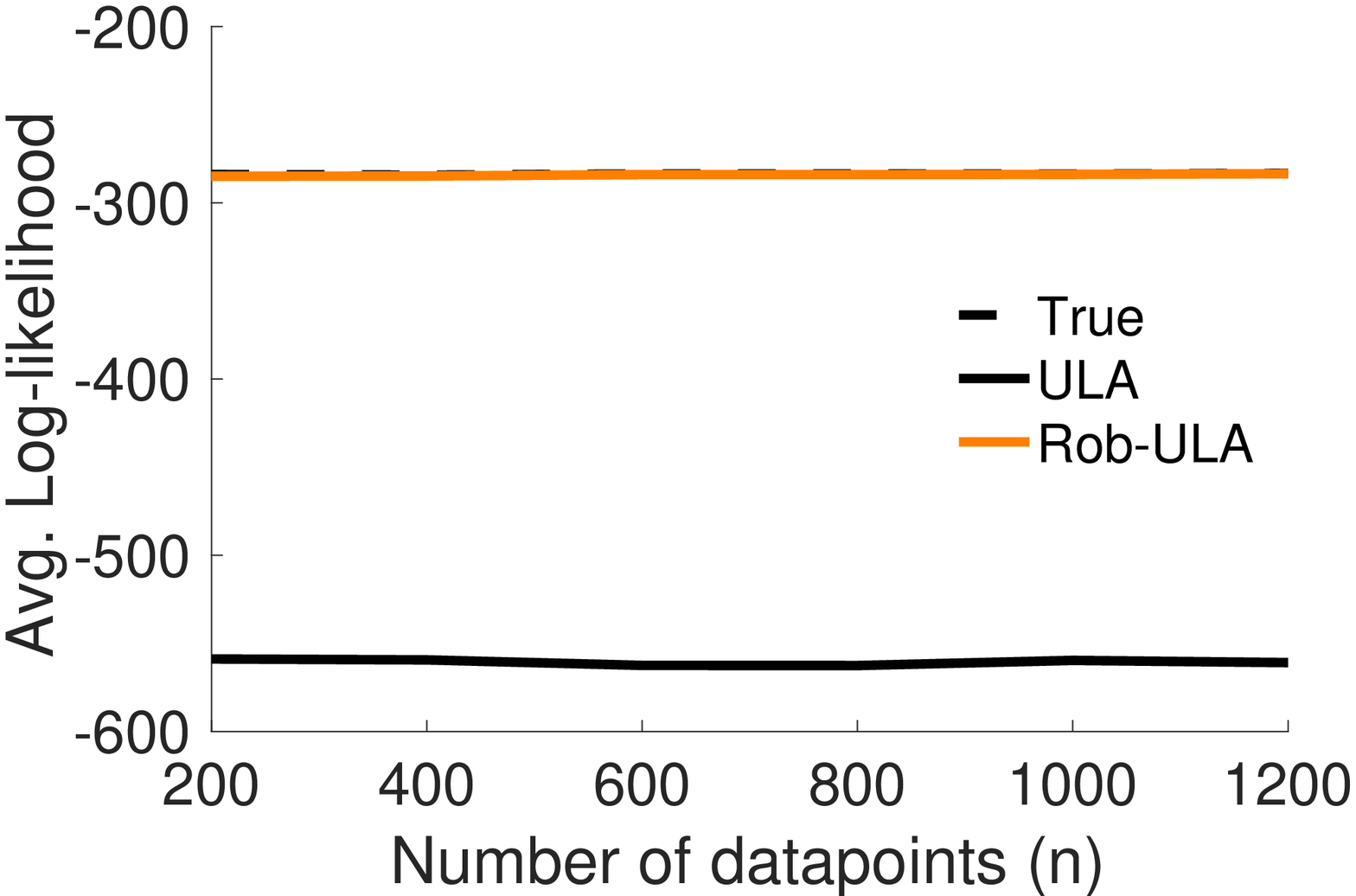}&
  \hspace{-4ex}
  \includegraphics[width=.33\textwidth]{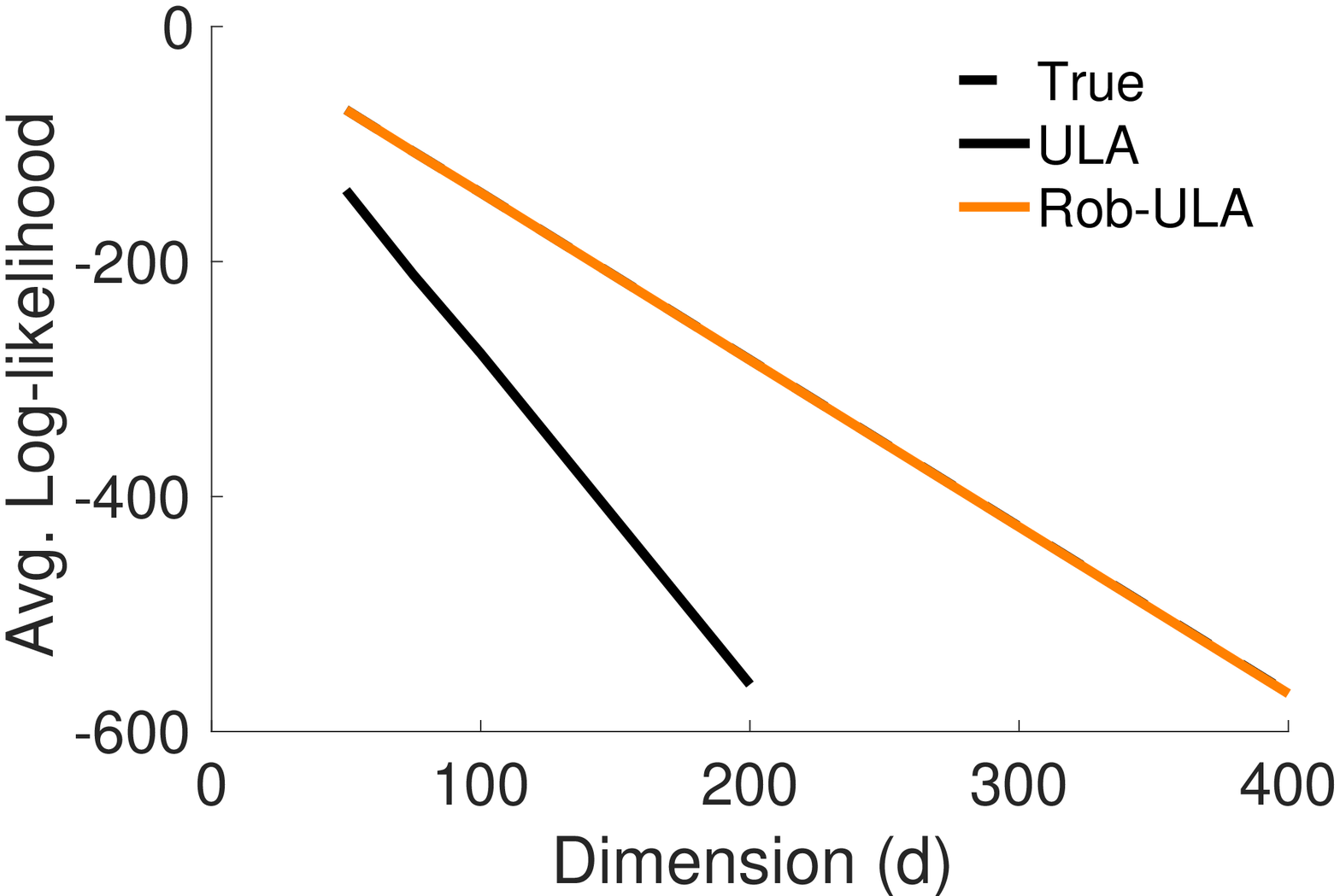}&
  \hspace{-4ex}
  \includegraphics[width=.33\textwidth]{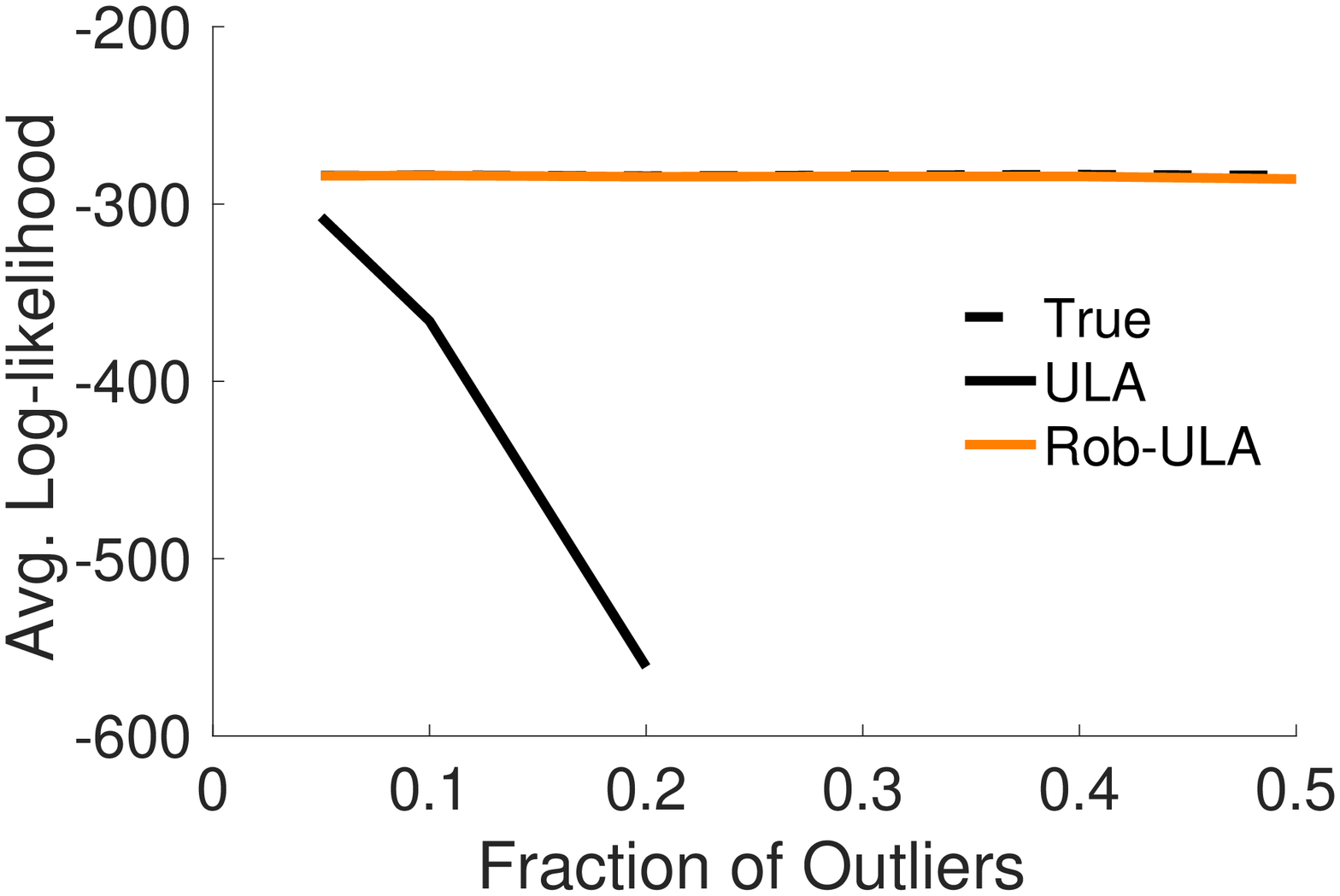}\\
(a)&(b)&(c)\vspace*{-5pt}
\end{tabular}
\caption{\small{Robust Bayesian Mean Estimation (Average log-likelihood): \alg has average log-likelihood close to the true underlying parameter. The log-likelihood values for ULA become large and negative for higher dimensions and fraction of outliers and do not show up in plots (b) and (c).}}\vspace*{-5pt}
  \label{fig:me_ll}
\end{figure}
In this section, we focus on experiments related to the robust Bayesian mean estimation problem described in Section~\ref{sec:rbme}. We begin by describing the experimental setup before proceeding to a discussion of the experimental findings. 
\paragraph{Experiment setup.} The mean vector $\param \in \real^d$ was sampled as a uniform distribution independently in each coordinate over the interval $[0, 1]$.  The clean samples $\z_i$ were then obtained independently from $\mathcal{N}(\param, I)$. The corrupted distribution $Q$ was chosen as the Gaussian distribution $\mathcal{N}(\param_c, I)$ with mean given by $\param_c = \param + \param_{\sf cor}$ where each entry of $\param_{\sf cor}$ is sampled i.i.d.\ from the uniform distribution over $[0,10]$. The default parameters were set as follows: number of clean samples $n_c = 1000$, dimension $d = 200$ and fraction of corruption $\epsilon = 0.2$. In every experiment, one of the parameters was varied keeping the others fixed. For both \alg and ULA, the burn-in period was set to $300$ samples and a total of $n_{\sf samp} = 1000$ samples were collected following the burn-in period. Each experiment was repeated for 10 runs and we report the mean performance of the methods across these runs. 

\paragraph{Recovery guarantees.} Figures~\ref{fig:me} and~\ref{fig:me_ll} compare the performance of the algorithms for the mean estimation problem. Figure~\ref{fig:me} shows the variation in parameter recovery error, $\|\hat{\param} - \param \|_2$ where $\hat{\param} = \frac{1}{n_{\sf samp}}\sum_{i} \param_i$ is the average of the collected samples. Figure~\ref{fig:me}(a) shows that with increasing number of data points, the error in \alg's estimate decreases until it starts to saturate. In addition, with an increasing dimension and fraction of outliers, the error in estimation for both \alg and ULA increases.  This is consistent with Theorem~\ref{thm:RULA_conv}. Figure~\ref{fig:me_ll} studies the variation in average log-likelihood of the average estimate $\hat{\param}$ on a held-out test set. We also plot the likelihood values obtained by plugging in the true parameter (dotted line). The samples output by \alg have likelihood values identical to the true underlying parameter while those of ULA are much lower for all experiments. Note that ULA fails to have finite likelihood values (up to Matlab precision) for dimensions $d > 200$ and fraction of corruption $\epsilon > 0.2$ and hence have been omitted from the curves. 

\subsubsection{Robust Bayesian Linear Regression}
\begin{figure}[t!]
  \centering\hspace*{-4ex}
  \captionsetup{font=small}
\begin{tabular}{ccc}
  \includegraphics[width=.33\textwidth]{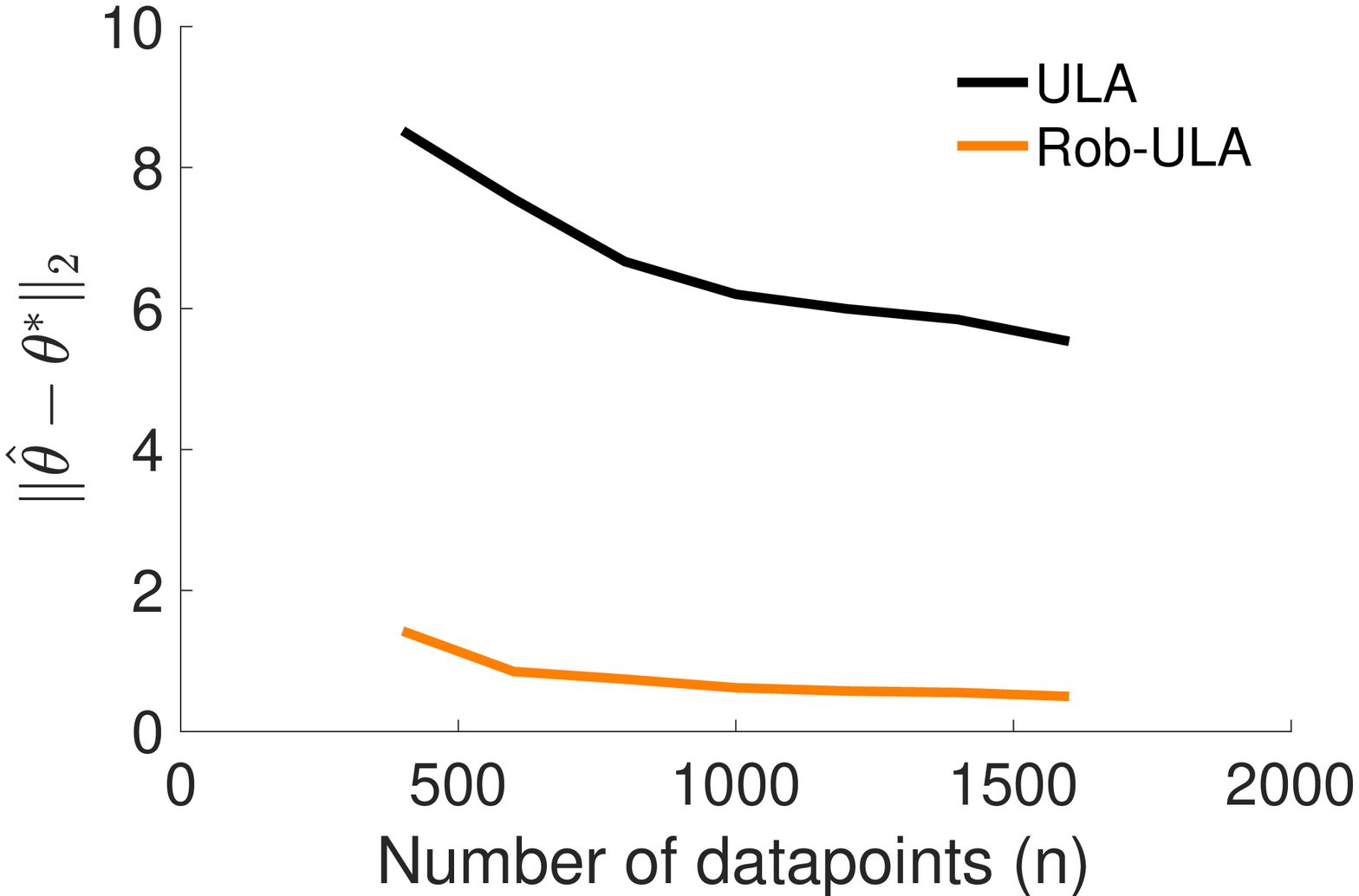}&
  \hspace{-4ex}
  \includegraphics[width=.33\textwidth]{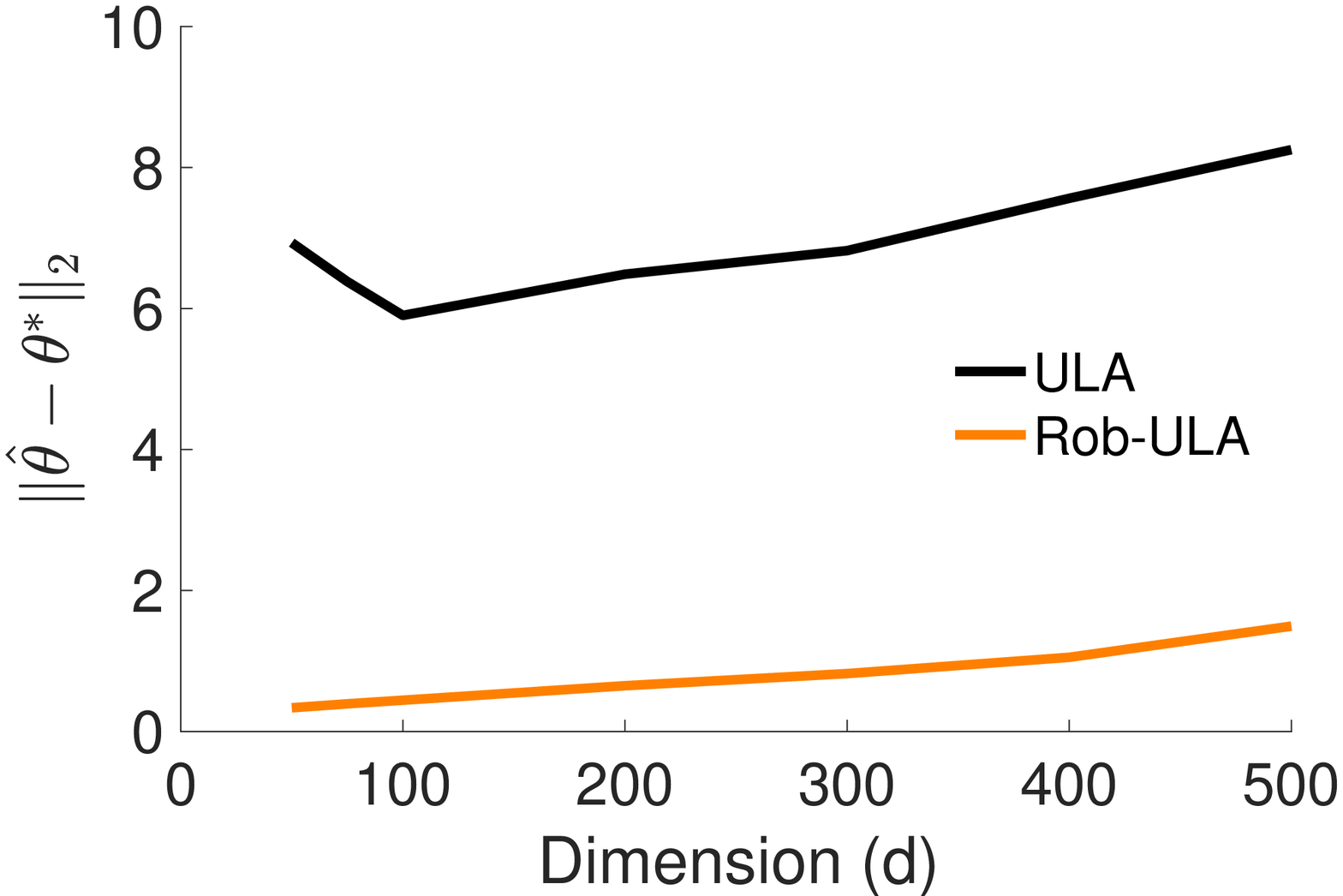}&
  \hspace{-4ex}
  \includegraphics[width=.33\textwidth]{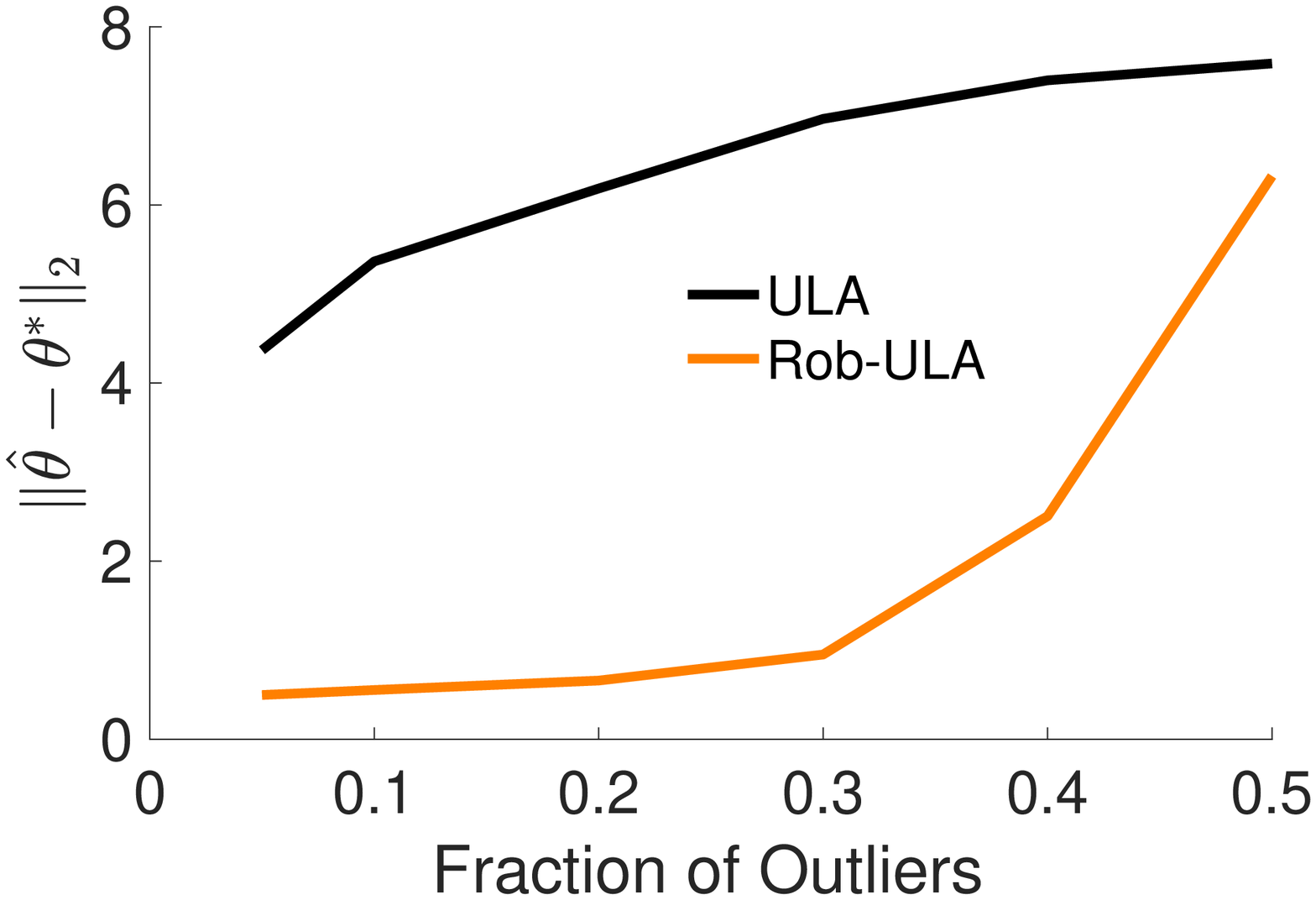}\\
(a)&(b)&(c)\vspace*{-5pt}
\end{tabular}
\caption{\small{Robust Bayesian Linear Regression (Parameter Estimation): \alg recovers the underlying parameter with smaller error as compared with the vanilla ULA. The recovery error increases with increasing dimension and fraction of outliers. For $\epsilon \approx 0.5$, the performance of \alg and ULA become quite similar in terms of recovery guarantees. }}\vspace*{-5pt}
  \label{fig:lr}
\end{figure}

In this section, we discuss the robust Bayesian linear regression problem described in Section~\ref{sec:rbls}. 

\paragraph{Experiment setup.} The true parameter $\param^*$ was selected similarly to the mean vector in the RBME experiments. The $\x$ were sampled i.i.d.\ from $\mathcal{N}(0, I)$ and the corresponding response variable were set as $y = x^\top \theta^* + \mathcal{N}(0, 1)$. For the corrupted distribution, each coordinate of the feature vector $x$ was sampled i.i.d.\ from a $\chi^2$ distribution and the corresponding response variables were set as $y =x^\top \theta^* + \text{Unif}[0,10] $. The default parameters were set as follows: number of clean samples $n_c = 1000$, dimension $d = 200$ and fraction of corruption $\epsilon = 0.2$. In every experiment, one of the parameters was varied keeping the others fixed. For both \alg and ULA, the burn-in period was set to $100$ samples and a total of $n_{\sf samp} = 300$ samples were collected. Each experiment was repeated for 10 runs and we report the mean performance of the methods across these runs. 

\paragraph{Performance guarantees.} Figure~\ref{fig:lr} shows the performance of \alg and ULA on the linear regression problem in terms of parameter recovery. Note that similar to the mean estimation setup, the error curves for \alg are lower than those for ULA as we vary number of data points, the dimensionality of the problem and the fraction of corruptions. The error shows a decreasing trend with increasing number of samples but increases with increasing dimension and fraction of outliers, showing that the robust problems indeed become harder in higher-dimensional spaces and with a larger fraction of outliers. 

\subsection{Real-world data sets}\label{sec:rwd}
In this section, we explore the performance of \alg on several real-world binary classification data sets obtained from the UCI repository~\citep{Dua17}. We use a logistic regression model.  While technically this model does not fall within the scope of Theorem~\ref{thm:RULA_conv} (primarily because of the strong-convexity assumption), we find that the experimental results are nonetheless consistent with the theory. We begin by providing details on the data sets used and then proceed to the experimental observations. For all the experiments in the section, the standard normal distribution was chosen as the prior. 

\paragraph{Data sets.} The logistic regression experiments were carried out with the following publicly available binary classification data sets: a) Astro~\citep{hsu2003}, b) Phishing~\citep{mohammad2012}, c) Breast-Cancer~\citep{mengersen1996}, d) Diabetes~\citep{Dua17} and e) German Credit~\citep{Dua17}. We normalized the features to scale between $[-1, 1]$ and used $70\%$ of the available data for training purposes and the remaining $30\%$ for testing purposes. For cases where we were required to tune hyperparameters, we used $20\%$ of the train data for validation purposes. Once the hyperparameters were fixed, we retrained the model with the complete training set. In order to understand the effects of corruptions in these data sets, we \emph{manually} add corruptions to the training subset in the form of label flips (for randomly chosen data points) for the Breast-Cancer, Diabetes and German Credit data set. The experiments with Astro and Phishing data sets are with the original \emph{uncorrupted} data sets.

\paragraph{Evaluation Metric.} For the above data sets, since the true underlying logistic parameter is unknown, we evaluate the algorithms using the log-likelihoods on the test set. For all data sets, we show two plots: plot (a) displays log-likelihood per data point in the test set, sorted in descending order for both algorithms and plot (b) shows a histogram of log-likelihoods. Plot (a) helps understand trends in prediction likelihoods by showing how the prediction quality degrades while Plot (b) provides an understanding of how the likelihoods concentrate. 

\subsubsection{Binary Classification}
\begin{figure}[t!]
  \centering\hspace*{-4ex}
  \captionsetup{font=small}
\begin{tabular}{cc}
  \includegraphics[width=.45\textwidth]{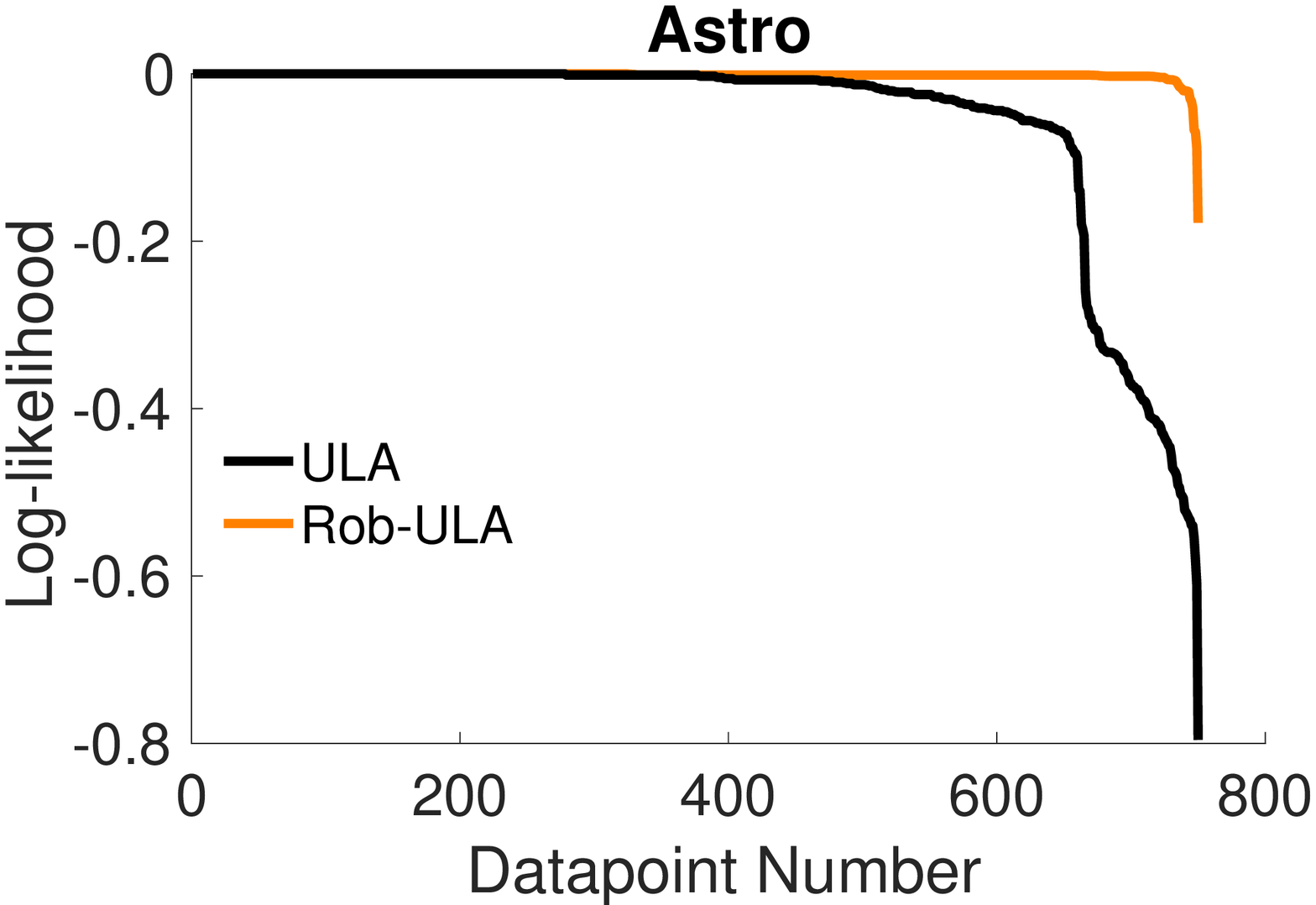}&
  \includegraphics[width=.45\textwidth]{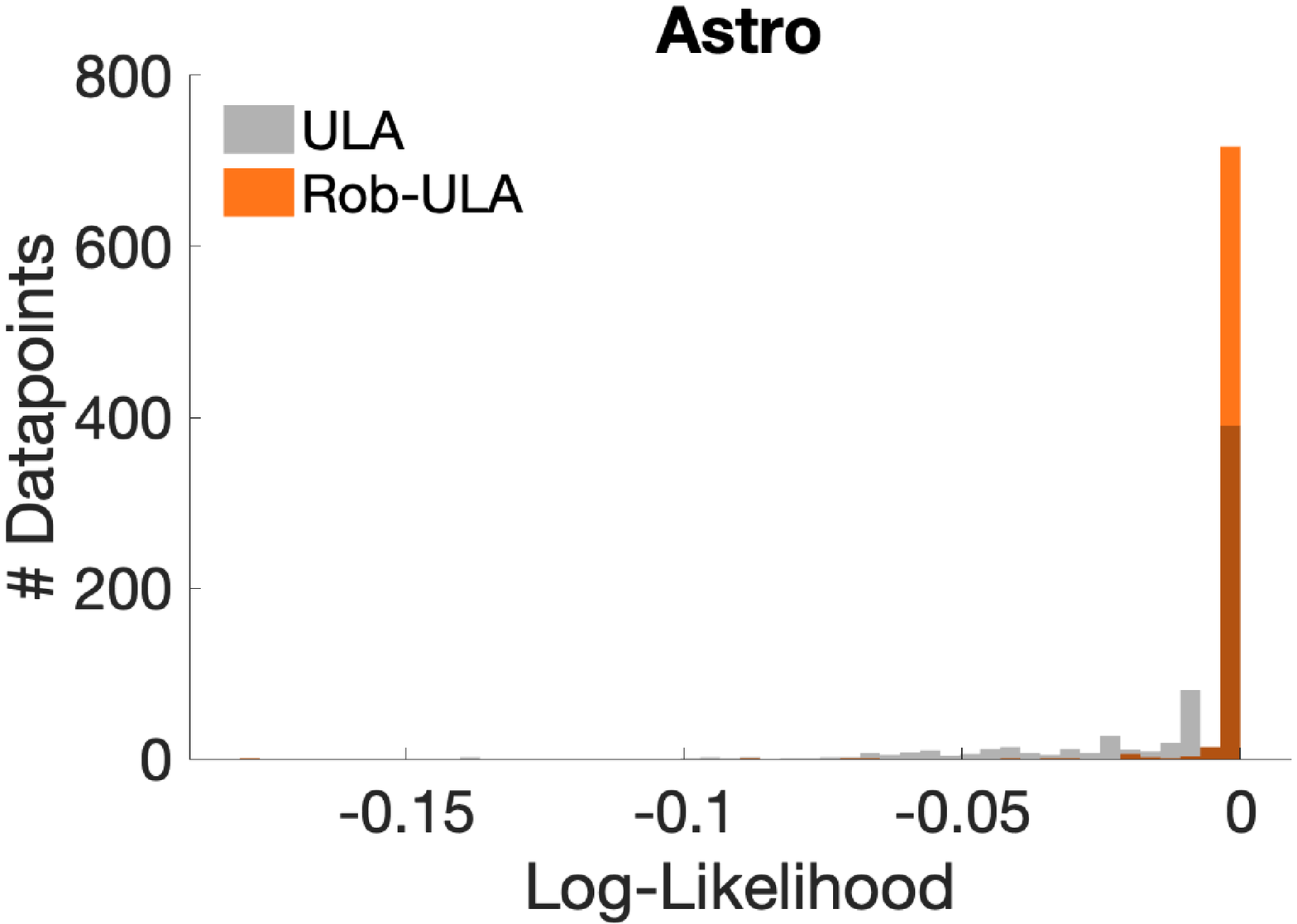}\\
(a)&(b)\vspace{5 pt}\\
\includegraphics[width=.45\textwidth]{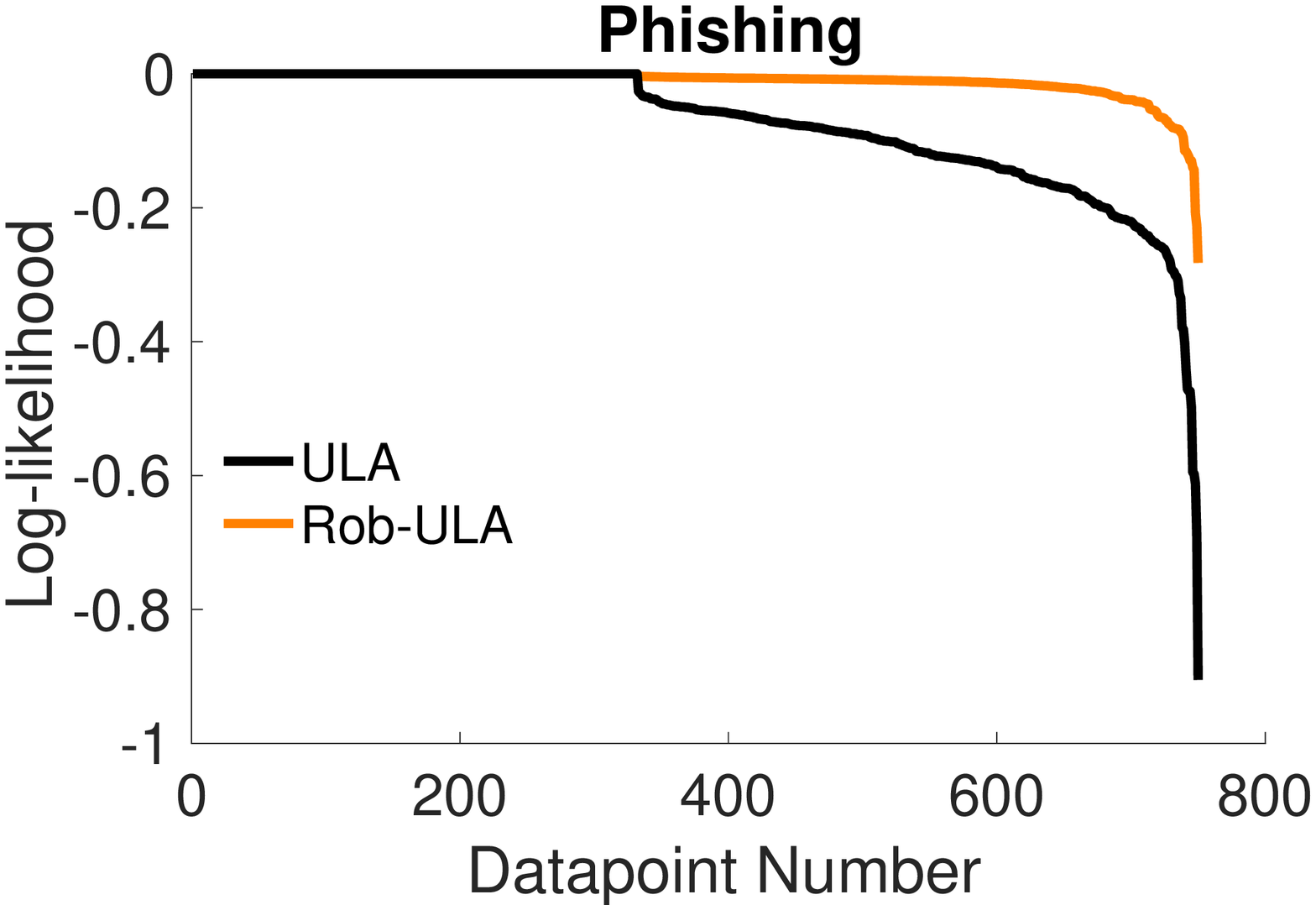}&
  \includegraphics[width=.45\textwidth]{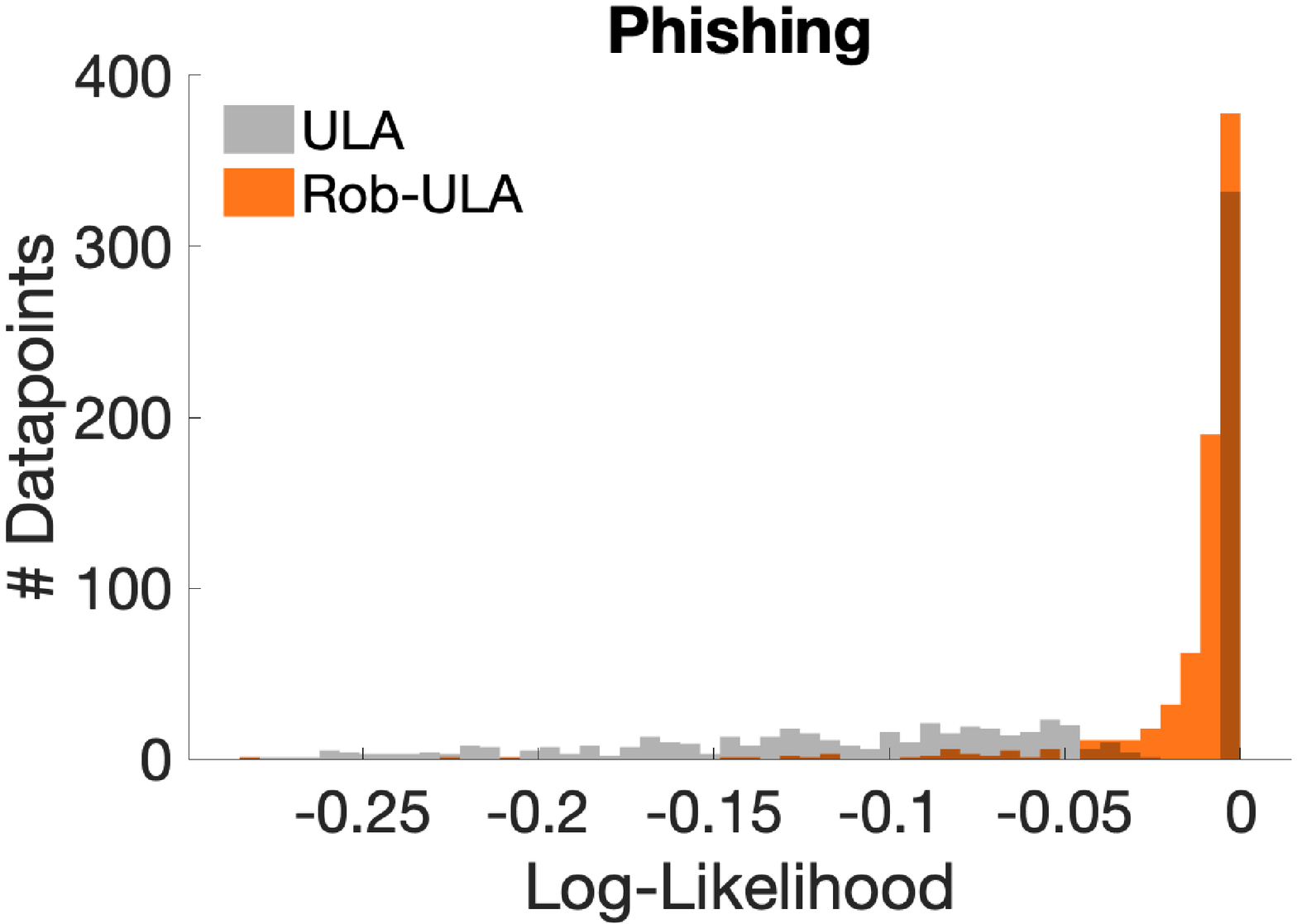}\\
(c)&(d)\vspace*{-5pt}
\end{tabular}
\caption{\small{Astro and Phishing Data set (Uncorrupted): \alg finds solutions which have better test log-likelihood performance across data points as compared to vanilla ULA. Surprisingly, even while ignoring a certain fraction of the data set, the performance of the \alg does not degrade in those regions of space.}}\vspace*{-5pt}
  \label{fig:astro}
\end{figure}

Figure~\ref{fig:astro} compares the performance of \alg on a binary classification task with no corruptions in the training set for the Astro and Phishing data sets. In both figures, \alg is seen to perform better than vanilla ULA: the histogram of likelihoods is more concentrated towards the origin. These data sets are not linearly separable and hence the logistic model may provide a poor fit to the data; \alg exploits this fact and focuses on a subset of points which it can fit well. This allows it to perform better for a larger range of points as compared to vanilla ULA. These experiments show that if there is model misspecification and the chosen model doesn't fit the complete data, the robust model might fit data selectively and give better confidence bounds for those data points. 

\subsubsection{Binary Classification with Label Flips}
\begin{figure}[t!]
  \centering\hspace*{-4ex}
  \captionsetup{font=small}
\begin{tabular}{cc}
  \includegraphics[width=.45\textwidth]{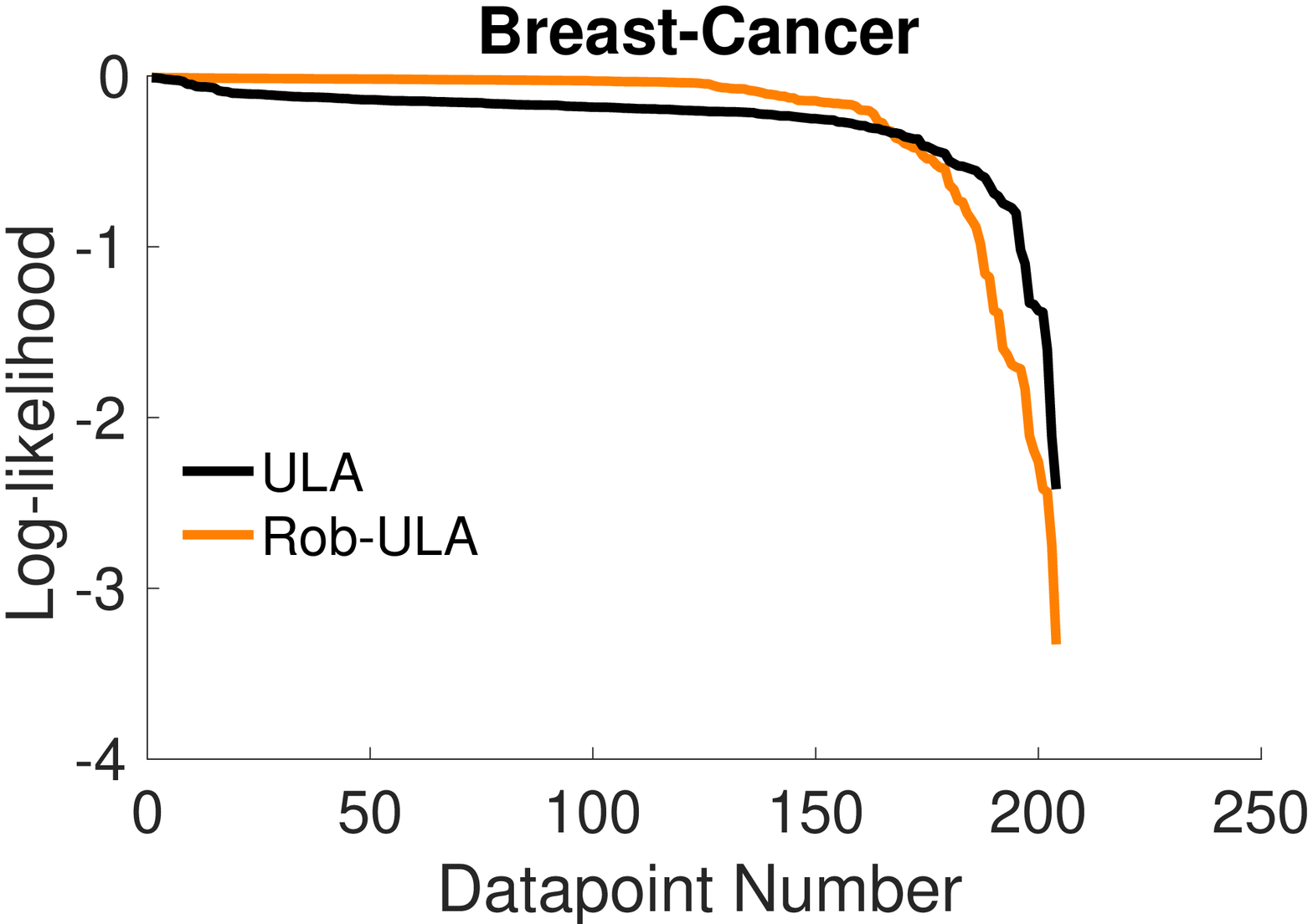}&
  \includegraphics[width=.45\textwidth]{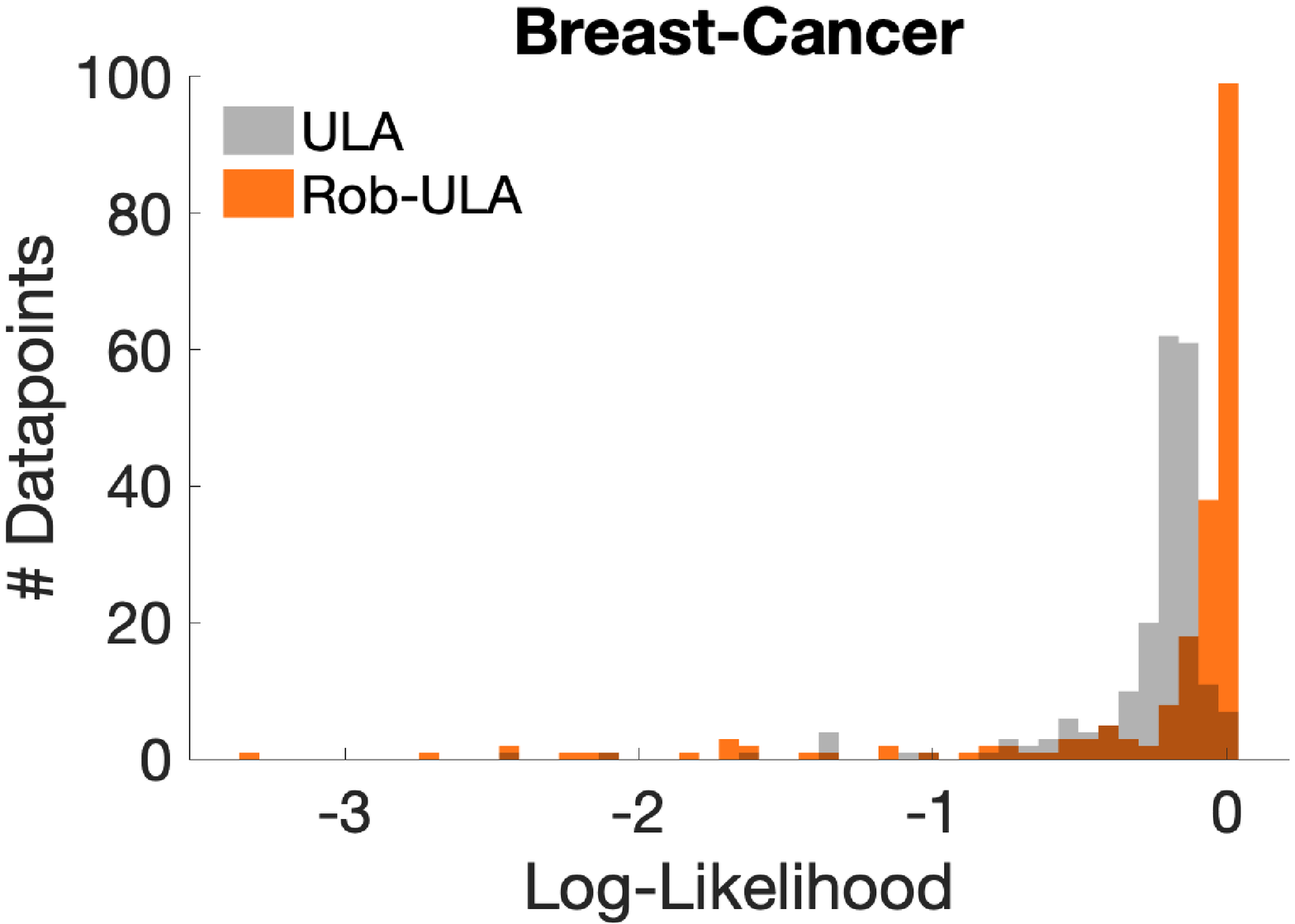}\\
(a)&(b)\vspace*{5 pt}\\
\includegraphics[width=.45\textwidth]{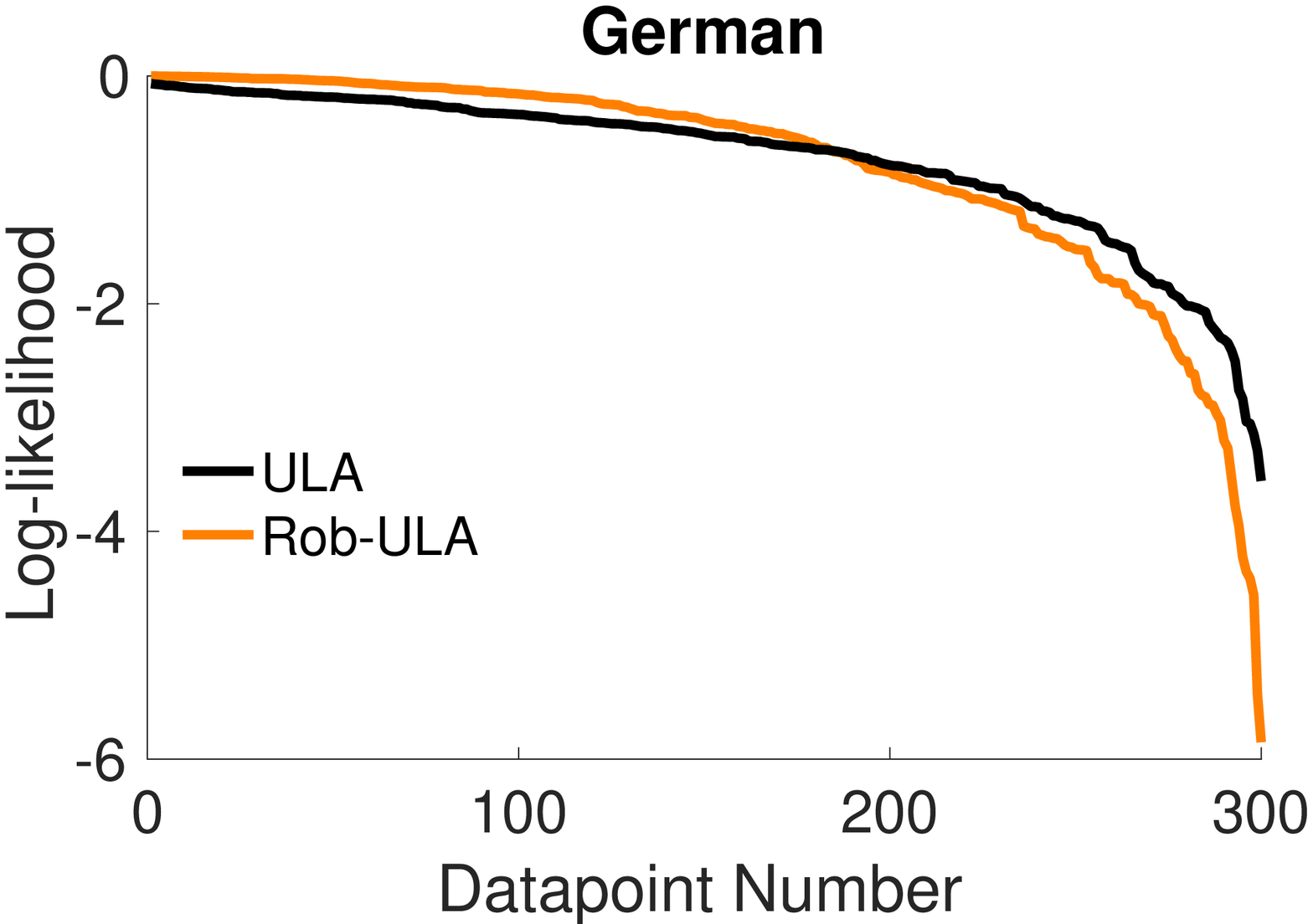}&
  \includegraphics[width=.45\textwidth]{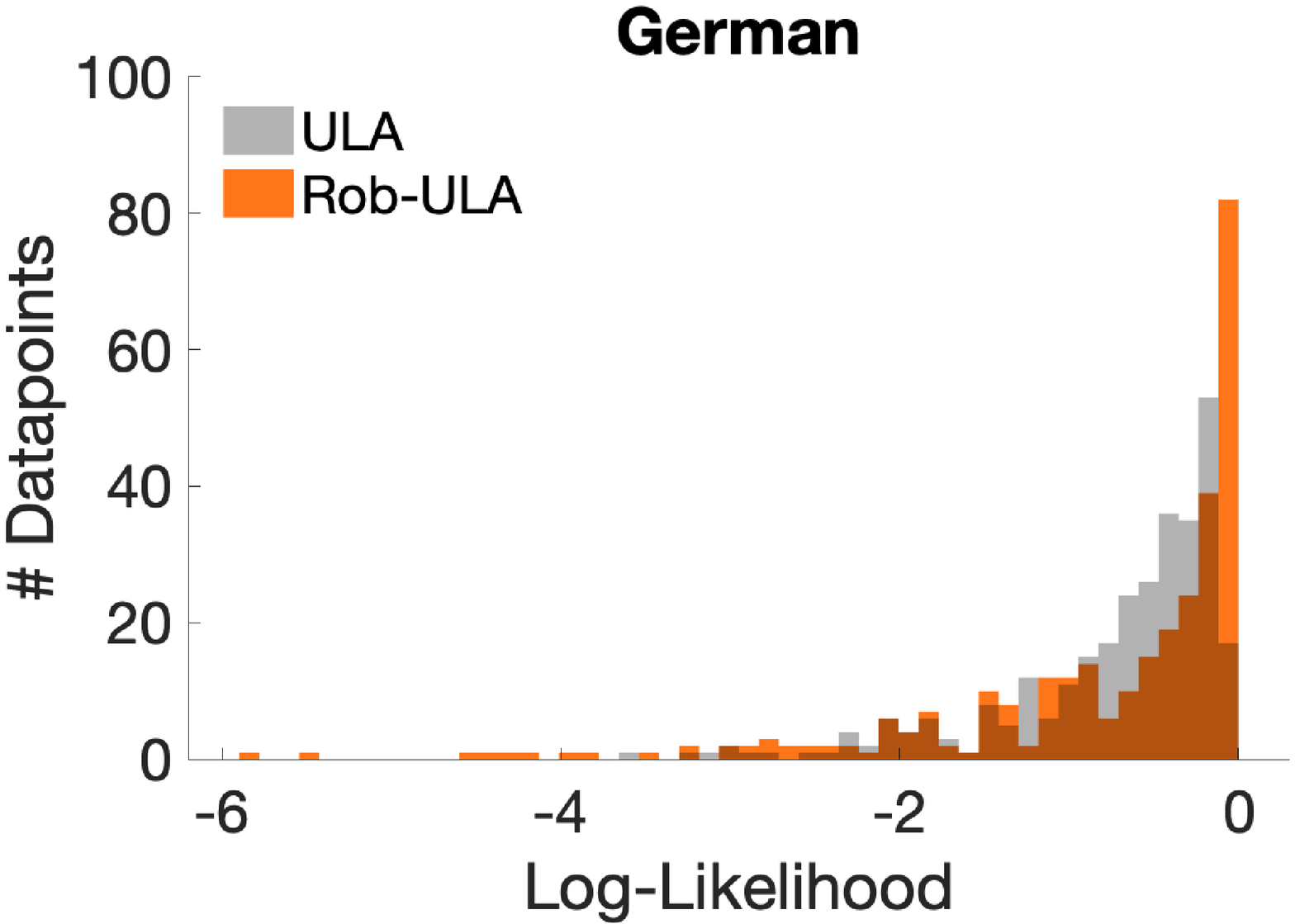}\\
(c)&(d)\vspace*{5pt}\\
 \includegraphics[width=.45\textwidth]{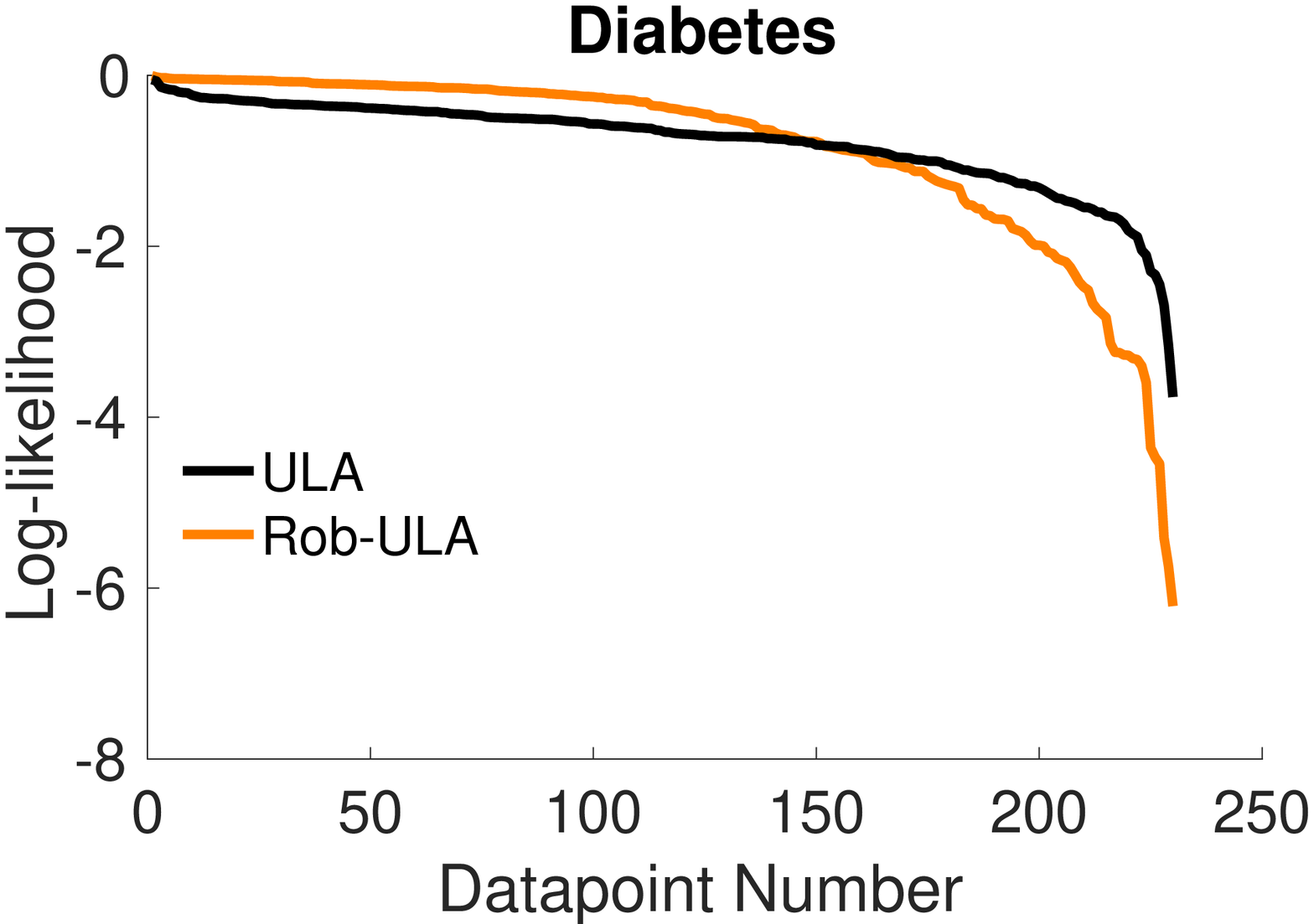}&
  \includegraphics[width=.45\textwidth]{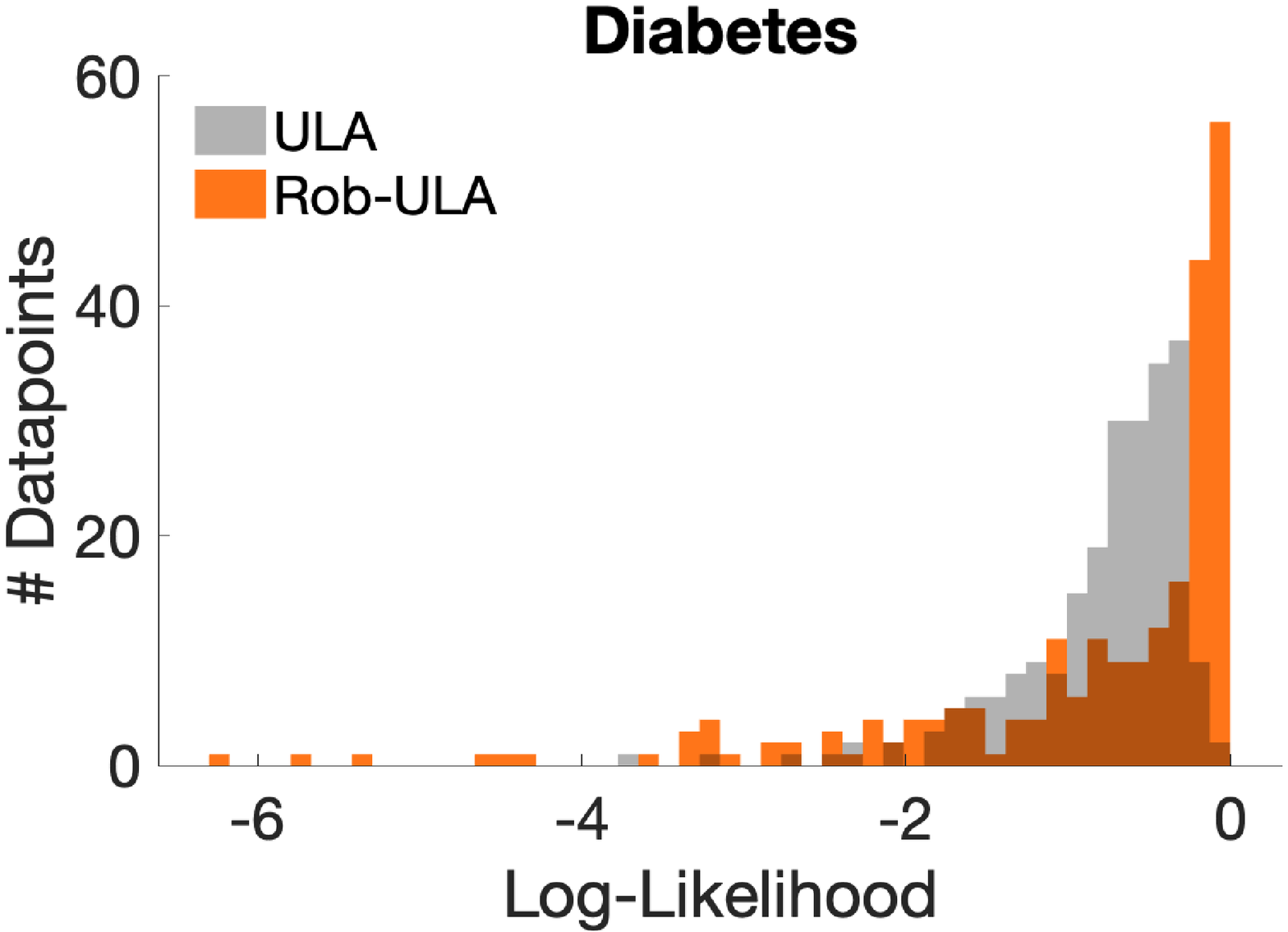}\\
(e)&(f)\vspace*{-5pt}
\end{tabular}
\caption{\small{Binary classification with label flips: Breast-Cancer ($\epsilon = 0.10$),  German Credit ($\epsilon = 0.10$) and Diabetes ($\epsilon = 0.15$). The plots show that while \alg is able to achieve high log-likelihood values for a vast majority of the data points, there is a small fraction of the points on which ULA performs better. This behavior can be attributed to \alg ignoring certain data points during its run and not generalizing well within the subspace spanned by them.}}\vspace*{-5pt}
  \label{fig:breastcancer}
\end{figure}


Figure~\ref{fig:breastcancer} shows the performance of \alg for the Breast-Cancer ($\epsilon = 0.10$),  German Credit ($\epsilon = 0.10$) and Diabetes ($\epsilon = 0.15$) data sets respectively. For these data sets, we manually added corruptions via label flips. For all three data sets, we see a similar trend in likelihood plots: \alg performs better than vanilla ULA for a majority of data points but its prediction quality decreases for the tail points. This can also be seen in the three histogram plots (part (b) of the respective figures) wherein \alg has likelihoods extending to larger negative values. This particular behavior can be attributed to the fact that \alg is unable to learn effective representations in the space where label flips were added since it chooses to ignore those data points. Hence, in the region corresponding to the uncorrupted points, \alg achieves higher likelihood values than vanilla ULA but for the corrupted regions, the performance of \alg degrades slightly. This behavior was consistently seen for varying levels of $\epsilon$ with minor shifts in the likelihood curves for different corruption levels. 
\section{Conclusions}\label{sec:conc}
We have discussed the problem of robustness to adversarial outliers in a Bayesian framework and proposed \alg, a robust extension of the classical Unadjusted Langevin algorithm. We obtain nonasymptotic convergence guarantees for \alg.

We identify multiple directions for future work. On the statistical side, it would be interesting to extend the robustness guarantees of \alg for statistical models which do not fall within the scope of our current assumptions, notably the case of nonconvex likelihood functions. On the computational side, an important question to understand is whether one can accelerate the convergence of \alg in the presence of outliers.
\section*{Acknowledgments} 
This work was done in part while KB, YM and PLB were visiting the Simons Institute for the Theory of Computing.
\newpage
\bibliographystyle{alpha}
\bibliography{research}
\newpage
\appendix
\section{Proof of Lemma~\ref{lem:rob_bound}}\label{app:robust}
We begin by obtaining a bound on the performance of Algorithm~\ref{alg:rge} in the one-dimensional setting and use this as a building block towards the proof for the general $d$-dimensional setting. For ease of exposition, we denote by $\mu_\param$ and $\hat{\mu}$ the mean gradient $\nabla U_\theta$ and its estimate $\widehat{\nabla} U_\param$ respectively. 
\subsection{Proof for 1 dimensional setting}
We begin by analyzing Algorithm~\ref{alg:rge} for the case when $d=1$.
\begin{lemma}\label{lem:1d_robust}
Let $P_\theta$ denote the empirical distribution on $\D_c$ in $\R$ with mean $\mu_\theta$, variance $\sigma$ and with fourth moment constant $C_4$. Let $\eta$ be the fraction of corruption in the samples in $\D$. Then Algorithm \ref{alg:rge} returns an estimate of the mean $\mu$ such that, for a universal constant $C$,
\begin{equation*}
  |\hat{\mu} - \mu_\theta| \leq  CC_4^\frac{1}{4}\sigma \eta^{\frac{3}{4}}.
\end{equation*}
\end{lemma}
\begin{proof}
 Let $I_{1-\eta}$ be the interval around the true sample mean $\mu_{\theta}$ containing $1-\eta$ fraction of $\D_c$. Using the bounded fourth moment assumption on $P_\theta$, we have that,
\begin{equation}\label{eq:length_bnd}
  \text{length}(I_{1-\eta}) \leq \frac{C_4^{\frac{1}{4}}\sigma}{\eta^{\frac{1}{4}}}.
\end{equation}
Let $\tilde{S}$ be the set of smallest interval containing $(1-\eta)^2$ fraction of all the points and let this interval be denoted by $\tilde{I}$. Now, we have that $\text{length}(\tilde{I})\leq \text{length}(I_{1-\eta})$ since the chosen interval has the smallest length. Further, $\tilde{S}$ contains at least $1-3\eta$ fraction of $\D_c$.

For any value of $\eta < \nicefrac{1}{6}$, we have that the intervals $\tilde{I}$ and $I$ must overlap. Therefore, adversarially corrupted points in the set $\tilde{S}$ are within distance $2\cdot \text{length}(I_{1-\eta})$ of the true parameter $\mu_\theta$. We now bound the deviation of the estimate $\hat{\mu}$ from $\mu_\theta$ by controlling the sources of error:
\paragraph{a. Error due to points in $\tilde{S}$ from $\D_a$:} Since there can be at most $\eta$ fraction of corrupted points in our selected sample and each of them within distance $2\cdot \text{length}(I_{1-\eta})$, their total contribution to the deviation is upper bounded by $2\eta\cdot \text{length}(I_{1-\eta})$.
\paragraph{b. Error due to points in $\tilde{S}$ from $\D_c$:} Define $\mathcal{A}$ to be the event that a point of $\D_c$ is present in the set $\tilde{S}$. From our discussion above, we have that $P(A) > 1-3\eta > 1/2$. Using Lemma 3.11 \cite{lai2016}, we have that there exits a constant $C$ such that
\begin{equation*}
  |\E{X|A} - \E{X}|\leq CC_4^\frac{1}{4}\sigma \eta^{\frac{3}{4}}.
\end{equation*}

Combining the analysis of parts (a) and (b) above, we get that,
\begin{equation*}
  |\hat{\mu} - \mu_\theta| \leq 2\eta\cdot \text{length}(I_{1-\eta}) + CC_4^\frac{1}{4}\sigma \eta^{\frac{3}{4}}.
\end{equation*}
Plugging in the bound for $\text{length}(I_{1-\eta})$ completes the proof.
\end{proof}

\subsection{Proof for d-dimensional setting}
We now proceed to prove the robustness properties of Algorithm \ref{alg:rge} for the general $d$-dimensional setting. Through the course of this section, we let $\tilde{S}$ be the set of points returned after the outlier truncation procedure with $\tilde{S}_c$ being the clean points and $\tilde{S}_a$ being the adversarially corrupted points. Also, we denote by $\mu_{\tilde{S}} := \text{mean}(\tilde{S})$, $\mu_{\tilde{S}_c} := \text{mean}(\tilde{S}_c)$ and $\mu_{\tilde{S}_a} := \text{mean}(\tilde{S}_a)$ the corresponding mean vectors of the relevant subsets.

\begin{lemma}\label{lem:subopt_est}
Let $P_\theta$ be the empirical distribution over $\D_c$ in $\R^d$ with mean $\mu_\theta$, covariance matrix $\Sigma_\theta$ and with fourth moment constant $C_4$. Let $\eta$ be the fraction of corrupted data points in $\D$. Then, we can obtain a vector $a \in \R^d$ such that for a constant $C$,
\begin{equation*}
  \|a-\mu_\theta \|_2 \leq CC_4^\frac{1}{4}\sqrt{\tr(\Sigma_\theta)}\eta^\frac{3}{4}.
\end{equation*}
\end{lemma}
\begin{proof}
  Let $e_1, \ldots, e_d$ be the $d$ canonical basis vectors. Projecting the problem onto these vectors and solving in each direction independently using the method for one dimension, we obtain the bound above by using Lemma \ref{lem:1d_robust} separately in each dimension.
\end{proof}

\begin{lemma}\label{lem:bnd_ptn_rad}
Let $\eta$ denote the fraction of outliers in $\D$. After the outlier truncation procedure, for every point in the returned set $\tilde{S}$, we have that for a constant $C$,
\begin{equation*}
\|x-\mu_\theta \|_2 \leq CC_4^{\frac{1}{4}}\left( \frac{\sqrt{d\|\Sigma_\theta \|_2}}{\eta^{\frac{1}{4}}}+\sqrt{\tr(\Sigma_\theta)}\eta^\frac{3}{4} \right).
\end{equation*}
\end{lemma}
\begin{proof}
Let $B^* = \mathcal{B}(\mu_\theta, r_1^*)$ for $r_1^* = \frac{CC_4^{\frac{1}{4}}}{\eta^{\frac{1}{4}}}\sqrt{d\|\Sigma_\theta \|_2}$ be the $\ell_2$ ball of radius $r_1^*$ around $\mu_\theta$. In order to bound the fraction of clean points in $B^*$, observe that,
\begin{equation}\label{eq:bnd_ball_good}
  P(\|x-\mu_\theta\|_2^2 \geq (r_1^*)^2) \leq \frac{\eta \E(\|x-\mu_\theta\|_2^4)}{C_4 d^2 \|\Sigma_\theta \|_2^2},
\end{equation}
where the probability is with respect to the empirical distribution over $\D$. Now, $\E(\|x-\mu_\theta\|_2^4 \leq d^2 \max_i \E((x - \mu_\theta)_i^4) \leq C_4 d^2 \|\Sigma_\theta \|_2^2$. Plugging this value in Equation \ref{eq:bnd_ball_good}, we get that at least $1-\eta$ fraction of the points of $\D_c$ lie in $B^*$.\\

Using Lemma \ref{lem:subopt_est}, we have that for $r_2^* = CC_4^\frac{1}{4}\sqrt{\tr(\Sigma_\theta)}\eta^\frac{3}{4}$, there are at least $(1-\eta)$ fraction of good points at a distance $r_1^* + r_2^*$ away from $a$. For $\eta < \nicefrac{1}{6}$, we have that the chosen ball of points around $\tilde{S}$ and $B^*$ must overlap. Therefore, minimum radius ball has radius at most $r_1^* + r_2^*$ and when combined with the bound from Lemma \ref{lem:subopt_est} and triangle inequality, we get that,
\begin{equation*}
  \|x-\mu_\theta \|_2 \leq CC_4^{\frac{1}{4}}\left( \frac{\sqrt{d\|\Sigma_\theta \|_2}}{\eta^{\frac{1}{4}}}+\sqrt{\tr(\Sigma_\theta)}\eta^\frac{3}{4} \right),
\end{equation*}
which completes the proof.
\end{proof}

\begin{lemma}\label{lem:shifts}
Let $\eta$ be the fraction of corrupted points in $\D$. Then we have that after the outlier truncation step of Algorithm \ref{alg:rge}, for a constant $C$,
\begin{equation*}
  \| \mu_{\tilde{S}_c} - \mu_\theta\|_2 \leq CC_4^{\frac{1}{4}}\eta^{\frac{3}{4}} \sqrt{\|\Sigma_{\theta}\|_2} \quad \text{ and }\quad \|\Sigma_{\tilde{S}_c}\|_2 \leq \|\Sigma_{\tilde{S}_c} - \Sigma_\theta\|_2 + \|\Sigma_\theta\|_2 \stackrel{\zeta_1}{\leq} (C\eta+1)\|\Sigma_\theta \|_2.
\end{equation*}
\end{lemma}
\begin{proof} We first consider the bound on the mean shift and then proceed with the bound on the covariance matrix.
\paragraph{Mean Shift Bound:} Let $\mathcal{A}$ be the event that a point $x\in \D_c$ is not removed by the outlier truncation procedure. Then using Lemma 3.11 [\cite{lai2016}] for $\eta < \nicefrac{1}{6}$ for the random variable $X = x^\top \frac{\mu_{\tilde{S}_c} - \mu_\theta}{\| \mu_{\tilde{S}_c} - \mu_\theta\|_2}$ for $x \sim \D_c$, we have that,
\begin{equation*}
  \| \mu_{\tilde{S}_c} - \mu_\theta\|_2 \leq CC_4^{\frac{1}{4}}\eta^{\frac{3}{4}} \sqrt{\|\Sigma_{\theta}\|_2}.
\end{equation*}

\paragraph{Covariance Matrix Bound:} Consider the following decomposition for bounding the spectral norm of $\Sigma_{\tilde{S}_c}$:
\begin{align*}
\|\Sigma_{\tilde{S}_c}\|_2 \leq \|\Sigma_{\tilde{S}_c} - \Sigma_\theta\|_2 + \|\Sigma_\theta\|_2 \stackrel{\zeta_1}{\leq} (C\eta+1)\|\Sigma_\theta \|_2,
\end{align*}
where $\zeta_1$ follows by using Corollary 3.13 \cite{lai2016} for the same event $\mathcal{A}$ as above in the mean shift bound.
\end{proof}

\begin{lemma}\label{lem:error_bnd_proj}
$P_W$ is the projection operator on the bottom $d/2$ eigenvectors of the matrix $\Sigma_{\tilde{S}}$. Then, for a constant $C$, we have that,
\begin{equation*}
  \|\eta P_W \delta_\mu\|_2^2 \leq \eta ((C\eta +1) + CC_4^{\frac{1}{2}}\eta^\frac{1}{2})\|\Sigma_\theta\|_2,
\end{equation*}
where $\delta_\mu := \mu_{\tilde{S}_a} - \mu_{\tilde{S}_c}$.
\end{lemma}
\begin{proof}
  Consider the matrix $\Sigma_{\tilde{S}}$. It can be decomposed as:
  \begin{equation*}
    \Sigma_{\tilde{S}} = \underbrace{(1-\eta)\Sigma_{\tilde{S}_c}}_{\Sigma_1} + \underbrace{\eta\Sigma_{\tilde{S}_a} + \eta(1-\eta)\delta_\mu\delta_\mu^\top}_{\Sigma_2}.
  \end{equation*}
By Weyl's inequality, we have that,
\begin{equation*}
  \lambda_{d/2}(\Sigma_{\tilde{S}}) \leq \lambda_1 (\Sigma_1) + \lambda_{d/2}(\Sigma_2).
\end{equation*}
We begin by first controlling the term $\lambda_{d/2}(\Sigma_2)$. We have that,
\begin{equation}\label{eq:control_p}
  \lambda_{d/2}(\Sigma_2) \leq \frac{\tr(\Sigma_2)}{d/2} \stackrel{\zeta_1}{\leq} C \eta \frac{(r_1^*)^2 + (r_2^*)^2}{d/2} \leq CC_4^{\frac{1}{2}}\|\Sigma_\theta \|_2\eta^\frac{1}{2},
\end{equation}
where $\zeta_1$ follows by using the fact that all selected points are in a ball of radius $r_1^* + r_2^*$ where $r_1^*$ and $r_2^*$ are as defined in Lemma \ref{lem:bnd_ptn_rad}. Next we consider the term $\lambda_1 (\Sigma_1)$ as follows,
\begin{equation}\label{eq:control_1}
  \lambda_1 (\Sigma_1) \stackrel{\zeta_1}{\leq} (1-\eta)(C\eta +1)\|\Sigma_\theta\|_2,
\end{equation}
where $\zeta_1$ follows from using the bound in Lemma \ref{lem:shifts}. Combining Equations \eqref{eq:control_p} and \eqref{eq:control_1}, we have that,
\begin{equation*}
  \lambda_{d/2}(\Sigma_{\tilde{S}}) \leq (1-\eta)(C\eta +1)\|\Sigma_\theta\|_2 + CC_4^{\frac{1}{2}}\|\Sigma_\theta \|_2\eta^\frac{1}{2}.
\end{equation*}
Using the fact that $P_W$ is the projection operator on the bottom $d/2$ eigenvectors of the matrix $\Sigma_{\tilde{S}}$, we have that,
\begin{equation*}
  P_W^\top\Sigma_{\tilde{S}}P_W \preceq ((1-\eta)(C\eta +1) + CC_4^{\frac{1}{2}}\eta^\frac{1}{2})\|\Sigma_\theta\|_2I.
\end{equation*}
Following some algebraic manipulation as in \cite{lai2016}, we obtain that
\begin{equation*}
\|\eta P_W \delta_\mu\|_2^2 \leq \eta ((C\eta +1) + CC_4^{\frac{1}{2}}\eta^\frac{1}{2})\|\Sigma_\theta\|_2.
\end{equation*}
which completes the proof.
\end{proof}

\section*{Proof of Lemma~\ref{lem:rob_bound}}
Let $\tilde{S}$ be the subset of samples returned by the outlier truncation procedure and let $\tilde{S}_{c}$ be the set of clean points contained in $\tilde{S}$. Then, we have,
\begin{align}\label{eq:err_decomp}
 \|\hat{\mu} - \mu_\theta \|_2^2 &\stackrel{\zeta_1}{=} \|P_W(\hat{\mu} - \mu_\theta) \|_2^2 + \| P_V(\hat{\mu} - \mu_\theta)\|_2^2\nonumber \\
 &\stackrel{\zeta_2}{\leq} 2 \|P_W(\hat{\mu} - \hat{\mu}_{\tilde{S}_{c}}) \|_2^2 + 2 \|P_W(\hat{\mu}_{\tilde{S}_{c}} - \mu_\theta) \|_2^2 + \| \hat{\mu}_V - P_V\mu_\theta)\|_2^2\nonumber \\
 &\stackrel{\zeta_3}{\leq} 2 \|P_W(\hat{\mu} - \hat{\mu}_{\tilde{S}_{c}}) \|_2^2 + 2 \|(\hat{\mu}_{\tilde{S}_{c}} - \mu_\theta) \|_2^2 + \underbrace{\| \hat{\mu}_V - P_V\mu_\theta)\|_2^2}_{(I)},
\end{align}
where $\zeta_1$ follows from the orthogonality of the spaces $V$ and $W$, $\zeta_2$ follows from using triangle inequality and $\zeta_3$ follows from contraction of projection operators. Note that $(I)$ is a problem defined on the subspace $V$ which is of ambient dimension $d/2$ and is solved recursively by Algorithm \ref{alg:rge}. Thus, one can recursively bound the overall error of the algorithm as,
\begin{equation}\label{eq:rec_bnd}
 \|\hat{\mu} - \mu_\theta \|_2^2 \leq \left( 2 \|P_W(\hat{\mu} - \hat{\mu}_{\tilde{S}_{c}}) \|_2^2 + 2 \|(\hat{\mu}_{\tilde{S}_{c}} - \mu_\theta) \|_2^2 \right) (1+\log d).
\end{equation}
Using Lemma \ref{lem:shifts} and Lemma \ref{lem:error_bnd_proj}, we can bound the above error as,
\begin{equation*}
\|\hat{\mu} - \mu_\theta \|_2 \leq CC_4^{\frac{1}{4}}\sqrt{\eta\log(d)\|\Sigma_{\theta}\|_2},
\end{equation*}
which completes the proof of the lemma.
\hfill{\qed}


\section{Convergence of \alg: Proofs for Auxiliary Lemmas }\label{app:lmc}
\begin{lemma}
For $\Theta_t$ following Eq.~\eqref{eq:disc_SDE_2}, if the initial iterate $\Theta_0\sim\mathcal{N}\left(0,\dfrac{1}{n\bar{L}}\mI\right)$, the fraction of corruption $\epsilon \leq \dfrac{\bar{m}^2}{4\Cr\Csa\log d}$, and scaled step-size $h = n \eta \leq\dfrac{1}{\bar{L}}$, then for all $k\in\mathbb{N}^+$,
\[
\E{\lrn{\vTheta_{k\stp} - \pari^*}_2^2} \leq \dfrac{4 \Cr\Csb}{\convav^2}\epsilon \log d + \dfrac{4d}{n \convav}.
\]
\label{lemma:variance}
\end{lemma}
\begin{proof}
Consider first the initial iterate for $k = 0$. The distribution $\p_0$ satisfy $\Ep{\param \sim\p_0}{\lrn{\param}_2^2} = \dfrac{d}{n\bar{L}} \leq \dfrac{4}{\bar{m}^2} C_{14}\epsilon \log d + \dfrac{4d}{n \bar{m}}$.
We will prove the lemma statement by strong induction.

In the induction hypothesis step, assume that for some $k\geq0$, for all $t=0, \stp, \ldots, k\stp$, $\Ep{\param\sim\p_t}{\lrn{\param}_2^2}\leq \dfrac{4}{\bar{m}^2} \Cr\Csb\epsilon \log d + \dfrac{4d}{n\bar{m}}$. We consider obtaining a bound on $\Ep{\param\sim\p_{(k+1)\eta}}{\lrn{\param}_2^2}$, where $\p_t$ follows Equation~\ref{eq:disc_SDE_2}, for $t\in(k\eta,(k+1)\eta]$ (denote $\tau=t-k\eta\in(0,h]$):
\begin{equation}
\Param_t = \Param_{k\eta} - \widehat{\nabla} \f(\Param_{k\eta}) \tau + \sqrt{2} (B_t - B_{k\eta}).
\end{equation}
Given the above equation, we consider the bound on $\E{\lrn{\Param_t - \pari}_2^2}$ for some \mbox{$t\in(k\eta,(k+1)\eta]$} as follows:
\begin{align*}
\E{\lrn{\Param_t - \pari}_2^2} &=
\E{\lrn{\left(\Param_{k\eta} - \pari\right) - \widehat{\nabla} \f(\Param_{k\eta}) \tau + \sqrt{2} (B_t - B_{k\eta})}_2^2}
\\ &=
\E{\lrn{\left(\Param_{k\eta} - \pari\right) - \widehat{\nabla} \f(\Param_{k\eta}) \tau}_2^2} + 2 d \tau.
\end{align*}
We next define $\nu=n\tau$ and obtain a bound on the term $\E{\lrn{\left(\Param_{k\eta} - \pari\right) - \widehat{\nabla} \f(\Param_{k\eta}) \tau}_2^2}$:
\begin{align*}
\E{\lrn{\left(\Param_{k\eta} - \pari\right) - \widehat{\nabla} \f(\Param_{k\eta}) \tau}_2^2} &=
\E{\lrn{\left(\Param_{k\eta} - \pari\right) - \dfrac{1}{n} \widehat{\nabla} \f(\Param_{k\eta}) \nu}_2^2}
\\ &=
\E{\lrn{\left(\vX_{k\eta} - \pari\right) - \dfrac{1}{n} \nabla \f(\vX_{k\eta}) \nu + \dfrac{1}{n} \nabla \f(\vX_{k\eta}) \nu - \dfrac{1}{n} \widehat{\nabla} \f(\vX_{k\eta}) \nu}_2^2}
\\ &= \E{\lrn{\left(\vX_{k\eta} - \pari\right) - \dfrac{1}{n} \nabla \f(\vX_{k\eta}) \nu}_2^2}
+ \frac{\nu^2}{n^2} \E{\lrn{\nabla \f(\vX_{k\eta}) - \widehat{\nabla} \f(\vX_{k\eta})}_2^2}
\\ &\quad
+ 2\nu \E{ \left\langle \left(\vX_{k\eta} - \pari\right) - \dfrac{1}{n} \nabla \f(\vX_{k\eta}) \nu , \dfrac{1}{n} \nabla \f(\vX_{k\eta}) - \dfrac{1}{n} \widehat{\nabla} \f(\vX_{k\eta}) \right\rangle }
\\ &\stackrel{\1}{\leq}
\left(1 + \bar{m}\nu\right) \E{\lrn{\left(\vX_{k\eta} - \pari\right) - \dfrac{1}{n} \nabla \f(\vX_{k\eta}) \nu}_2^2}\\
&\quad+ \left(\dfrac{\nu}{\bar{m}} + \nu^2\right) \dfrac{1}{n^2} \E{\lrn{\nabla \f(\vX_{k\eta}) - \widehat{\nabla} \f(\vX_{k\eta})}_2^2},
\end{align*}
where $\1$ follows by an application of Cauchy–Schwarz inequality. Next, using the assumption on the robust estimation of the gradient from Theorem~\ref{thm:RULA_conv}, we have that
\begin{equation*}
\dfrac{1}{n^2} \E{\lrn{\nabla \f(\vX_{k\eta}) - \widehat{\nabla} \f(\vX_{k\eta})}_2^2} \leq \Cr\Csa\epsilon \log d \cdot \E{\lrn{\vX_{k\eta} - \pari}^2} + \Cr\Csb\epsilon \log d,
\end{equation*}
and further simplifying the above using Lemma~\ref{lemma:aux}, we obtain 
\begin{equation}\E{\lrn{\left(\vX_{k\eta} - \vx^*\right) - \dfrac{1}{n} \nabla U(\vX_{k\eta}) \nu}_2^2} \leq (1-\bar{m}\nu)^2 \E{\lrn{\vX_{k\eta} - \vx^*}^2}.
\end{equation}
Therefore, since $\nu\leq\dfrac{1}{\bar{L}}$ and the corruption factor$\epsilon\leq\dfrac{\bar{m}^2}{4\Cr\Csa\log d}$, we have that
\begin{align*}
\E{\lrn{\left(\vX_{k\eta} - \pari\right) - \dfrac{1}{n} \widehat{\nabla} \f(\vX_{k\eta}) \nu}_2^2} &\leq \left(1-\bar{m}^2\nu^2\right) \left(1-\bar{m}\nu\right) \E{\lrn{\vX_{k\eta} - \pari}^2}\\
&\quad+ \left(\dfrac{\nu}{\bar{m}} + \nu^2\right)\left( C_{13}\epsilon \log d \E{\lrn{\vX_{k\eta} - \pari}^2} + C_{14}\epsilon \log d \right)
\\ &\leq \left(1-\bar{m}\nu + \dfrac{2\nu}{\bar{m}} C_{13}\epsilon \log d\right) \E{\lrn{\vX_{k\eta} - \pari}^2}
+ \dfrac{2\nu}{\bar{m}} C_{14}\epsilon \log d
\\ &\leq \left(1-\dfrac{\bar{m}\nu}{2}\right) \E{\lrn{\vX_{k\eta} - \pari}^2}
+ \dfrac{2\nu}{\bar{m}} C_{14}\epsilon \log d,
\end{align*}
where we have defined the constants $C_{13}:= \Cr\Csa$ and $C_{14}:= \Cr\Csb$. Using the above bounds, we have that
\begin{equation*}
\E{\lrn{\vX_t - \pari}_2^2}
\leq \left(1-\dfrac{\bar{m}\nu}{2}\right) \E{\lrn{\vX_{k\eta} - \vx^*}^2}
+ \dfrac{2\nu}{\bar{m}} C_{14}\epsilon \log d + \dfrac{2\nu}{n} d.
\end{equation*}
Note that $\nu\leq\dfrac{1}{\bar{L}}\leq\dfrac{1}{\bar{m}}$ and  $\E{\lrn{\vX_{k\eta} - \pari}^2} \leq \dfrac{4 C_{14}}{\bar{m}^2}\epsilon \log d + \dfrac{4d}{n \bar{m}} = \dfrac{2}{\bar{m}} \left(\dfrac{2}{\bar{m}} C_{14}\epsilon \log d + \dfrac{2}{n} d\right)$. Combining these, we can obtain the final bound stated in the lemma as follows:
\begin{align*}
\E{\lrn{\vX_t - \pari}_2^2}
&\leq \left(1-\dfrac{\bar{m}\nu}{2}\right) \dfrac{2}{\bar{m}} \left(\dfrac{2}{\bar{m}} C_{14}\epsilon \log d + \dfrac{2}{n} d\right)
+ \dfrac{2\nu}{\bar{m}} C_{14}\epsilon \log d + \dfrac{2\nu}{n} d
\\ &\leq \dfrac{2}{\bar{m}} \left(\dfrac{2}{\bar{m}} C_{14}\epsilon \log d + \dfrac{2}{n} d\right)
\\ &= \left(\dfrac{4}{\bar{m}^2} C_{14}\epsilon \log d + \dfrac{4d}{n \bar{m}}\right),
\end{align*}
for any $t\in(k\eta,(k+1)\eta]$. This concludes the proof.
\end{proof}

We now prove the Lemma~\ref{lemma:aux} which was used in the proof of Lemma~\ref{lemma:variance}. 
\begin{lemma}
For $\tau \leq \dfrac{1}{\bar{L}}$, and $\dfrac{1}{n} \nabla \f(\param)$ being $\bar{m}$ strongly convex and $\bar{L}$ Lipschitz smooth, we have
\begin{equation*}
\lrn{ (\param - \pari) - \dfrac{1}{n} \nabla \f(\param)\nu}_2^2
\leq
(1-\bar{m}\nu)^2 \lrn{\param-\pari}_2^2.
\end{equation*}
\label{lemma:aux}
\end{lemma}
\begin{proof}
To bound $\lrn{ \param - \dfrac{1}{n} \nabla \f(\param) \nu }_2^2$, we consider the following function: $F(\param) = \dfrac12 \lrn{\param}_2^2 - \dfrac{1}{n} \f(\param) \nu$. First note strong convexity and Lipschitz smoothness of $\dfrac{1}{n} \nabla \f(\param)$ implies that $ \bar{m} \mI \preceq \dfrac{1}{n} \nabla^2 \f(\param) \preceq \bar{L} \mI$.
Thus with $\nu \leq \dfrac{1}{\bar{L}}$, we have that
\begin{equation*}
(1-\nu \bar{L}) \mI \preceq \nabla^2 F(\param) \preceq (1-\nu \bar{m}) \mI \quad \forall \;\param \in\mathbb{R}^d. \end{equation*}
Note that the point $\pari$ satisfies  $\nabla \f(\pari) = 0$. Using this we have that:
\begin{align*}
\lrn{ (\param - \pari) - \dfrac{1}{n} \nabla \f(\param) \nu}_2^2 &= 
\lrn{ \left(\param - \dfrac{1}{n} \nabla \f(\param) \nu\right) - \left(\pari - \dfrac{1}{n} \nabla \f(\pari) \nu\right)}_2^2
\\ &=
\lrn{ \int_0^1 \nabla^2 F\left( \lambda \param + (1-\lambda)\pari \right)\rd\lambda (\param-\pari)}_2^2\\
&\stackrel{\1}{\leq} (1-\bar{m}\nu)^2 \lrn{\param-\pari}_2^2,
\end{align*}
where $\1$ follows from the Lipschitz-smoothness of $F$.
\end{proof}

\begin{lemma}[Bound on Initial Error]
If we let the initial iterate $\Theta_0$ have distribution given by
\[\displaystyle
\p_0(\param) = \left(\dfrac{L}{2\pi}\right)^{d/2} \exp\left(-\dfrac{L}{2}||\param||^2\right)
\]
and $\p^*$ following Assumptions~\ref{A1}--\ref{A2}, then the initial error \emph{$\kldiv{\p_0}{\p^*}$} is bounded as:
\emph{\begin{equation*}
{\kldiv{\p_0}{\p^*}} = \int \p_0(\param)\ln\left(\dfrac{\p_0(\param)}{\p^*(\param)}\right)\rd \param \leq
\dfrac{d}{2}\ln\dfrac{L}{m}. \label{eq:KL_init}
\end{equation*}}
\label{lemma:initial_dist}
\end{lemma}

\begin{proof}
We want to bound ${\kldiv{\p_0}{\p^*}} = \displaystyle\int \p_0(\param)\ln\left(\dfrac{\p_0(\param)}{\p^*(\param)}\right)\rd \param$, where $\p^*(\param)\propto e^{-\f(\param)}$.
First define $\bar{\f}(\param) = \f(\param) - \f(\param^*)$, where $\param^*$ is the minimum of $\f$.
Then
\[
\p^*(\param) = \frac{\exp\left(-\bar{\f}(\param)\right)} {\int \exp\left(-\bar{\f}(\param)\right) \rd \param}.
\]
By Assumptions~\ref{A1} and~\ref{A2}, we have that $\dfrac{m}{2}\|\param\|^2\leq\bar{\f}(\param)\leq\dfrac{L}{2}\|\param\|^2$, $\forall \param\in\mathbb{R}^d$. Therefore,
\begin{align*}
-\ln p^*(\param) &= \bar{\f}(\param) + \ln{\int \exp\left(-\bar{\f}(\param)\right) \rd \param} \\
&\stackrel{\1}{\leq} \dfrac{L}{2}\|\param\|^2 + \ln \int \exp\left(-\dfrac{m}{2}\|\param\|^2\right) \rd \param \\
&= \dfrac{L}{2}\|\param\|^2 + \dfrac{d}{2}\ln\dfrac{2\pi}{m},
\end{align*}
where $\1$ follows from using the Lipschitz-smoothness of $\bar{\f}$. Hence
\begin{equation*}
- \int \p_0(\param) \ln p^*(\param) \rd \param
\leq \dfrac{d}{2}\ln\dfrac{2\pi}{m} + \dfrac{d}{2}.
\end{equation*}
We can also calculate similarly that
\begin{equation*}
\int \p_0(\param) \ln p_0(\param) \rd \param
= -\dfrac{d}{2}\ln\dfrac{2\pi}{L} - \dfrac{d}{2}.
\end{equation*}
Combining the above, we get that
\begin{equation*}
\kldiv{\p_0}{\p^*} = \int \p_0(\param) \ln p_0(\param) \rd \param - \int \p_0(\param) \ln p^*(\param) \rd \param \leq \dfrac{d}{2}\ln\dfrac{L}{m}
= \dfrac{d}{2}\ln\dfrac{\bar{L}}{\bar{m}}.
\end{equation*}
This concludes the proof of the statement.
\end{proof}

\section{Proofs for Mean Estimation and Regression}\label{app:cons}
\subsection{Proof of Corollary~\ref{cor:rbme}}\label{sec:bme}
Throughout this proof, we condition on the high probability event described by Equation~\eqref{eq:bnd_ep_m}. We proceed to obtain a bound on the strong-convexity parameter $m$ and Lipschitz smoothness parameter $L$ for the function $\f(\vtheta)$ defined above in Equation \eqref{eq:u_rbme}. The gradient $\nabla \f(\vtheta)$ is given by
\begin{equation*}
  \nabla \f(\vtheta) = \Sigma_0^{-1}(\vtheta - \vtheta_0) +\sum_{i \in \D_c} \Sigma^{-1}(\vtheta - \z_i),
\end{equation*}
and the corresponding hessian $\nabla^2 \f(\vtheta)$ is given by,
\begin{equation*}
  \nabla^2 \f(\vtheta) = \Sigma_0^{-1} + |\D_c|\cdot\Sigma^{-1}.
\end{equation*}
Since both the matrices $\Sigma_0$ and $\Sigma$ are positive-definite, we have the following bounds for the parameters $m$ and $L$:
\begin{equation}\label{eq:lm_rbme}
  \underbrace{\left(\frac{n(1- \bar{\epsilon})}{\lambda_\text{max}(\Sigma)} + \frac{1}{\lambda_\text{max}(\Sigma_0)}\right)}_{m}I \preceq \nabla^2 \f(\vtheta) \preceq \underbrace{\left(\frac{n(1- \underline{\epsilon})}{\lambda_\text{min}(\Sigma)} + \frac{1}{\lambda_\text{min}(\Sigma_0)}\right)}_{L}I.
\end{equation}

We now obtain a bound on the covariance $\Sigma_\theta$ of the gradients $\nabla g_i(\param)$ using the empirical distribution of the points in $\D_c$ as follows:
\begin{align}\label{eq:bnd_cov_rbme}
  \Sigma_\vtheta &= \frac{1}{|\D_c|}\sum_{i \in \D_c} \left(\nabla g_i(\param) - \mathbb{E}_{z \sim_u \D_c }\nabla g_z(\param)\right)\left(\nabla g_i(\param) - \mathbb{E}_{z \sim_u \D_c }\nabla g_z(\param)\right)^\top\nonumber \\
  &= \Sigma^{-1}\cdot \frac{1}{|D_c|}\sum_{i\in \D_c}(\z_i - \mu_z)(\z_i - \mu_z)^\top\nonumber\\
  &= \Sigma^{-1}\Sigma_z,
\end{align}
where $\z \sim_u \D_c$ denotes data sampled uniformly from the data set $\D_c$. Thus, from the above equation, we get that $\Csa = 0$ and $\Csb = \frac{\lmx(\Sigma_z)}{\lmn(\Sigma)}$. Plugging in these values in Theorem~\ref{thm:RULA_conv} gives us the desired bound.
\hfill{\qed}

\section{Proof of Corollary~\ref{cor:rbls}}\label{sec:blr}
We condition on the high probability event described by Equation \eqref{eq:bnd_dc_lr}. We begin by obtaining a bound on the strong-convexity parameter $m$ and Lipschitz smoothness parameter $L$ for the function $\f(\vtheta)$ defined above in Equation \eqref{eq:u_rblr_m}. The gradient $\nabla \f(\vtheta)$ is given by
\begin{equation*}
  \nabla_\theta \f(\vtheta) = \Sigma_0^{-1}(\vtheta - \vtheta_0) +\frac{1}{\sigma^2}\sum_{i \in \D_c} \left(\x_i\ip{\x_i}{\vtheta} - y_i\x_i\right),
\end{equation*}
and the corresponding hessian $\nabla^2 \f(\vtheta)$ is given by,
\begin{equation*}
  \nabla^2 \f(\vtheta) = \Sigma_0^{-1} + \frac{|\D_c|}{\sigma^2}\cdot\wtS_x.
\end{equation*}
Since both the matrices $\Sigma_0$ and $\wtS_x$ are positive-definite, we have the following bounds for the parameters $m$ and $L$:
\begin{equation}\label{eq:lm_rbme}
  \underbrace{\left(n(1- \bar{\epsilon}){\lambda_\text{min}(\wtS_x)} + \frac{1}{\lambda_\text{max}(\Sigma_0)}\right)}_{m}I \preceq \nabla^2_\vtheta \f(\vtheta) \preceq \underbrace{\left(n(1- \underline{\epsilon}){\lambda_\text{max}(\wtS_x)} + \frac{1}{\lambda_\text{min}(\Sigma_0)}\right)}_{L}I.
\end{equation}
We now proceed to obtain the bound on the spectral norm of the covariance matrix $\Sigma_\vtheta$ of the gradients $\nabla g_i(\param) = \x_i(\ip{\x_i}{\vtheta} - y_i)$ using the empirical distribution of the points in $\D_c$. We use the notation $\ex_{\D_c}$ to denote $\ex_{(\x,y) \sim_u \D_c}$, the sampling of pairs $(\vx,y)$ uniformly from the clean data set $\D_c$.
\begin{align}\label{eq:decomp_sigma_rblr}
  \|\Sigma_\vtheta \|_2 &= \sup_{v \in \mathbb{S}^{d-1}} v^\top \left(\ex_{\D_c} \left[\nabla g_{i}(\param) \nabla g_i(\param)^\top\right] - \ex_{\D_c}\left[ \nabla g_i(\param)\right]\ex_{\D_c}\left[ \nabla g_i(\param)\right]^\top \right)v\nonumber\\
  &\stackrel{\1}{\leq} \sup_{v \in \mathbb{S}^{d-1}} v^\top \left(\ex_{\D_c} \left[\nabla g_i(\param) \nabla g_i(\param)^\top\right] \right)v\nonumber \\
  &= \sup_{v \in \mathbb{S}^{d-1}} \ex_{\D_c}\left[(v^\top \x)^2 (\ip{x}{\vtheta}-y)^2 \right] \nonumber\\
  &\stackrel{\2}{\leq } \sup_{v \in \mathbb{S}^{d-1}} \sqrt{\ex_{\D_c}\left[ (v^\top x)^4\right]} \sqrt{\ex_{\D_c}\vphantom{\sum}\left[ (\ip{\x}{\vtheta} - y)^4\right]},
\end{align}
where $\1$ follows from the fact that $\ex_{\D_c}\left[ \nabla g_i(\param)\right]\ex_{\D_c}\left[ \nabla g_i (\param)\right]^\top \succeq 0$ and $\2$ follows from the Cauchy-Schwarz inequality. We now obtain a bound on the two expectations in the above equation.

\textbf{Bound on $\ex_{\D_c}\left[ (v^\top \x)^4\right]$}: This term can be bounded using bounded fourth moment assumption (Assumption \ref{A4}) as follows:
\begin{equation}\label{eq:part1}
\ex_{\D_c}\left[ (v^\top \x)^4\right] \leq C_{x,4}\left(\ex_{x\sim_u\D_c}\left[(v^\top\x)^2\right]\right)^2\leq C_{x,4} \|\wtS_x \|_2^2.
\end{equation}

\textbf{Bound on $\ex_{\D_c}\vphantom{\sum}\left[ (\ip{\x}{\vtheta} - y)^4\right]$}: We simplify this term by using the modelling assumption on the data $(\x,y)$ and then proceed to bound this using the $c_r$-inequality.
\begin{align}
  \ex_{\D_c} \left[(\ip{\x}{\vtheta} - y)^4 \right] &= \ex_{\D_c} \left[(\ip{\x}{\dtheta} - z)^4 \right]\nonumber\\
  &\stackrel{\1}{\leq} 8\cdot\left(\ex_{\D_c}\left[ (\ip{\x}{\dtheta})^4\right] + \ex_{\D_c}\left[ z^4\right]\right)\nonumber \\
  &\stackrel{\2}\leq 8\cdot\left(C_{x,4}\|\wtS_x\|_2^2\|\dtheta\|_2^4 + \ex_{\D_c}\left[ z^4\right]\right)\nonumber,
\end{align}
where $\Delta_\vtheta := \vtheta -\vtheta^*$, $\1$ follows from using the $c_r$-inequality $\ex |X+Y|^r\leq 2^{r-1}\left(\ex|X|^r + \ex|Y|^r\right)$ and $\2$ follows from Assumption 2 on bounded fourth moment of the covariates. Using Lemma~\ref{lem:gauss_fourth_moment} for bounding the fourth moment of the noise variables, we have with probability at least $1-\delta$,
\begin{equation}\label{eq:part2}
  \ex_{\D_c} \left[(\ip{\x}{\vtheta} - y)^4 \right] \leq 8\cdot\left(C_{x,4}\|\wtS_x\|_2^2\|\dtheta\|_2^4 + 3\sigma^4 + \frac{C_{z,4}\sigma^4}{\sqrt{n}}\log^2\left( \frac{e^2}{\delta}\right)\right).
\end{equation}
Substituting the bounds obtained in Equations~\eqref{eq:part1} and \eqref{eq:part2} in Equation~\eqref{eq:decomp_sigma_rblr}, along with an application of triangle inequality, we have that with probability at least $1-\delta$,
\begin{align}
  \|\Sigma_\vtheta \|_2 &\leq \underbrace{\sqrt{C_{x,4}}\|\wtS_x \|_2 + 2\sqrt{8C_{x,4}}\cdot\|\wtS_x\|_2\cdot\|\vtheta^* - \vtheta_{\text{reg}}\|_2^2  +  \frac{(8C_{z,4})^{\frac{1}{2}}\cdot\sigma^2}{n^{\frac{1}{4}}}\log\left( \frac{e^2}{\delta}\right) + \sqrt{24}\sigma^2}_{\Csb}\nonumber\\
  &\quad+ \underbrace{2\sqrt{8C_{x,4}}\cdot\|\wtS_x\|_2}_{\Csa}\cdot \|\vtheta - \vtheta_{\text{reg}}\|_2^2.
\end{align}
One can now use the above bounds in conjunction with the values of $\lipav$ and $\convav$ from Equation~\eqref{eq:lm_rbme} to obtain the final result.
\hfill{\qed}\\

\noindent The following lemma obtains a concentration bound for the fourth moment of a Guassian random variable and can be obtained by appropriate instantiation of the Hypercontractivity Concentration Inequality (Theorem 1.9) by Schudy and Sviridenko~\citep{schudy2012}.
\begin{lemma}[Concentration Bound for Gaussian Fourth Moment]
\label{lem:gauss_fourth_moment}
Let $z_1, z_2, \ldots, z_n$ be i.i.d. random variable sampled from $\mathcal{N}(0,1)$. Then, there exists a universal constant $C_{z,r}$ such that for any $\epsilon >0$, we have that,
\begin{equation*}
  \Pr\left( \left\vert \frac{1}{n}\sum_{1=1}^n z_i^4 - \ex [z^4]\right\vert \geq \epsilon \right) \leq e^2 \exp\left(\frac{-n^{\frac{1}{4}}\epsilon^{\frac{1}{2}}}{C_{z,4}\cdot \sigma^2} \right).
\end{equation*}
\end{lemma}

\end{document}